\documentclass[11pt]{article}
\usepackage[margin=1in]{geometry}
\RequirePackage{amsthm,amsmath,amsfonts,amssymb}
\usepackage{enumitem}
\usepackage{float}
\RequirePackage[authoryear]{natbib}
\RequirePackage[colorlinks,linkcolor=blue, citecolor=blue,urlcolor=blue]{hyperref}
\RequirePackage{graphicx}
\usepackage{cleveref}
\crefname{assumption}{Assumption}{Assumptions}
\crefname{algorithm}{Algorithm}{Algorithms}
\crefname{example}{Example}{Examples}
\crefname{table}{Table}{Tables}
\crefname{subsection}{Section}{Subsections}
\crefname{section}{Section}{Sections}
\crefname{proposition}{Proposition}{Propositions}
\crefname{lemma}{Lemma}{Lemmas}
\crefname{theorem}{Theorem}{Theorems}
\usepackage[format=plain,justification=raggedright,singlelinecheck=false]{caption}
\usepackage{subcaption}
\usepackage{tabularx}
\usepackage{bm}
\usepackage{algorithm}
\usepackage{algpseudocode}

\def\cZ{{\cal Z}}

\def\cC{{\cal C}}

\def\cS{{\cal S}}

\def\reals{{\mathbb R}}

\def\eps{{\varepsilon}}

\def\E{{\mathbb E}}

\def\L0{{L_i}}

\def\<{\langle}
\def\>{\rangle}

\def\limsup{\rm{lim \, sup}}

\def\F{{\sf F}}

\def \Tr{{\rm Trace}}

\def\F{{\sf F}}

\def\sT{{\sf T}}

\def\v*{v_i}
\def\T*{T_i}

\def\u*{u_i}
\def\F*{F_i}

\newcommand{\mtx}[1]{\bm{#1}}

\def\bR{{\mtx{R}}}

\def\th{\theta}

\def\cX{\mathcal{X}}

\def\l1u{W}

\newcommand{\ajcomment}[1]{}

\makeatletter
\newcommand{\labitem}[2]{%
\def\@itemlabel{\text{#1}}
\item
\def\@currentlabel{#1}\label{#2}}
\makeatother

\DeclareMathAlphabet{\mathpzc}{OT1}{pzc}{m}{it}

\def\cP{\mathcal{P}}

\theoremstyle{plain}

\newtheorem{theorem}{Theorem}[section]
\newtheorem{lemma}[theorem]{Lemma}
\newtheorem{corollary}[theorem]{Corollary}
\newtheorem{assumption}{Assumption}

\theoremstyle{definition}

\newcommand{\st}{\;\mathsf{s.t.}\;}
\newcommand{\hems}{\mathsf{HEMS}}
\newcommand{\affil}[1]{{\normalfont\small\sffamily #1}}
\newcommand{\email}[1]{{\normalfont\small\texttt{#1}}}

\begin{document}

\title{Learning to Price Electricity for\\ Optimal Demand Response}


\author{ \begin{tabular}{c@{\hspace{2em}}c} Jing Shang & Mohammad Mehrabi \\ \affil{Stanford University} & \affil{Stanford University} \\ \email{jshang21@stanford.edu} & \email{mehrabi@stanford.edu}\\ \\ Xinyang Zhou & Mahmoud Saleh \\ \affil{National Laboratory of the Rockies} & \affil{Prime Coalition} \\ \email{xinyang.zhou@nlr.gov} & \email{moody.saleh@primecoalition.org} \\ \\ Andrey Bernstein & Stefan Wager \\ \affil{National Laboratory of the Rockies} & \affil{Stanford University} \\ \email{aberenstein@gmail.com} & \email{swager@stanford.edu} \end{tabular} }


\date{Draft version \today}
\maketitle

\begin{abstract}
There is considerable interest in using time-varying electricity prices to shape consumer demand response, and better align energy demand with renewable production. However, optimal prices generally vary over time in response to complex signals such as weather forecasts, sunrise/sunset times, and day-of-week patterns; and existing methods are not able to make efficient use of such rich contextual information. Here, we propose a neural-network-based algorithm for contextual energy pricing, modeling pricing as a Stackelberg game and leveraging a mean-field solution representation from Mehrabi et al.~(2024). The approach learns constrained mappings from contextual features to feasible price signals. We validate our approach by simulating the energy grid in several US cities, and show that incorporating contextual information can considerably increase the value of the demand response programs.
\end{abstract}


\section{Introduction}

The share of renewables in electricity generation is growing, and rapidly ushering us towards a new operational regime in power systems. Wind and solar are now central sources of electricity in the United States, with their combined share of generation rising from less than $1\%$ in 2005 to $17\%$ in 2025; this share rises to $19\%$ when small-scale solar panels are also included \citep{EIA_WindSolar_2026}. This growth is driven by a number of factors. Renewables (e.g., utility-scale solar PV) are now among the lowest-cost new generation resources, and they can often be deployed faster than traditional fuel-based power plants \citep{Lazard_LCOE_2025}. In addition, the direct generation emissions of renewables are substantially lower than those of fossil-based resources, making renewables economically and operationally more attractive sources of new generation.
 
Despite the growing interest in renewable resources, their continued large-scale deployment leads to significant operational challenges. The key challenge is that the timing of renewable production is idiosyncratic and generally not controllable, and thus generally misaligned with demand \citep{EIA_DuckCurve_2023}. Because electricity is expensive to store, this supply-demand misalignment already prevents us from efficiently using all produced energy. For example, in 2024, CAISO indeed has curtailed approximately 3,000 GWh of solar supply in California \citep{EIA_CAISO_Curtailment_2025}; and we expect such challenges to only get more acute over time \citep{CAISO_FlexCapacity_2026}.
At the same time, utilities face growing pressure on their systems and the costs of operation from rapid load growth, including large new loads such as data centers. DOE reported in 2024 that U.S. data-center electricity use is expected to double or triple by 2028, while broader electrification of buildings and transportation is expected to place additional demands on distribution infrastructure \citep{DOE_DataCenterDemand_2024}.

Given this context, there is considerable interest in leveraging flexibility in residential consumer demand to better align consumption with renewable supply. This approach is mainly motivated by two primary reasons: 1) the residential sector approximately covers one-third of U.S. retail electricity consumption (so managing this demand might have a significant effect), and 2) many residential loads, including HVAC systems, water heating, appliance use, and electric-vehicle (EV) charging, have flexibility in timing \citep{EIA_ElectricityUse_2026}. For example, a household can pre-cool its home during a sunny afternoon and therefore lower air conditioning usage later in the evening (the peak-hour time), while keeping the indoor temperature within their desired comfort range. Similarly, an electric vehicle can be charged anytime before a certain deadline. In both cases, electricity consumption can be shifted across hours without requiring the consumer to sacrifice their desired energy goal.

Existing literature has considered the design of various mechanism for demand response, including via the use of time-varying prices \citep{adelman2019dynamic,crew1995theory,mehrabi2024optimalmechanismsdemandresponse,zhou2017incentive,zhou2019online}.
The existing work, however, has largely focused on a time-homogeneous setting,\footnote{Or quasi time-homogeneous, where different price schedules are used in, e.g., winter vs.~summer, but the price schedule is fixed within each season.} where the goal is to learn a fixed price schedule that can then be deployed every day. This is acceptable when the power system conditions are expected to be stable, or when the optimal price schedule does not vary significantly with exogenous factors. In grids with high penetration of renewables, however, this assumption can often be violated and result in poor demand management by fixed price schedules.

To illustrate this phenomenon, consider a simple example in which some days are sunny and others are cloudy. On sunny days, we have excess capacity provided by solar resources in the middle of the day; whereas on cloudy days we do not.
In \autoref{fig:heterogeneous_optimal_price}, we verify that optimal time-varying price schedules for sunny vs.~cloudy days are different; and that deploying prices optimized for a cloudy day on a sunny day will lead to poor performance.\footnote{Specifically, we observe the emergence of a ``duck curve'' in the net demand profile (i.e., total demand minus renewable supply). The duck curve is a recurring pattern observed in solar-rich electricity markets: Net demand experiences a deep dip around midday, when solar output is high, followed by a steep ramp-up in the early evening as solar supply declines while demand peaks---so that the overall profile resembles a duck.}
Clearly, in such a setting, it is desirable to have access to context-responsive prices that respond to day-by-day varying system conditions such as the weather.

\begin{figure}[t]
    \includegraphics[width=1\linewidth]{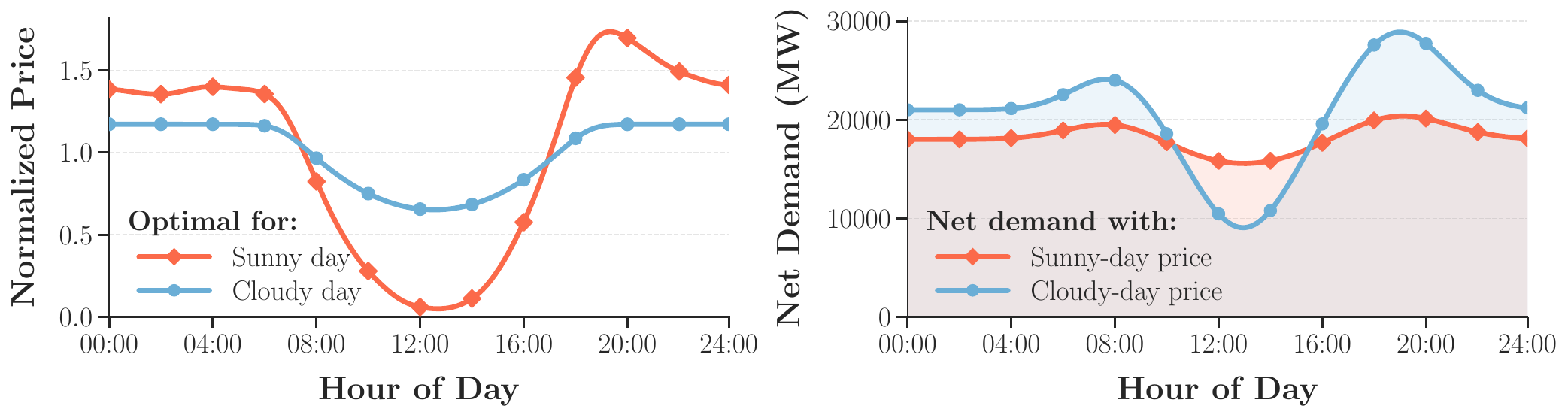}
    \caption{Illustration of the value of context-aware pricing. Left: Optimal price profiles for sunny and cloudy weather conditions obtained by minimizing a grid-risk objective based on net demand. The objective combines the $l_2$ norm and the sum of squared differences between consecutive periods, capturing overall demand magnitude and temporal fluctuations. Right: Net demand on a single (sunny) day when both price profiles from the left panel are applied. The sunny-day optimal price produces a more stable demand profile, while the cloudy-day price, applied to the same sunny day, yields poor performance with a pronounced duck curve.}
    \label{fig:heterogeneous_optimal_price}
\end{figure}

The main contribution of this paper is a neural-network based online learning algorithm
for contextual energy pricing. We consider a sequential model where, each day, the planner
observes a set of day-specific features (e.g., weather forecasts or demand signals); deploys
time-varying prices for the upcoming day; and then observes resulting consumer behavior.
We model consumer response via the indifference-set approach of
\citet{mehrabi2024optimalmechanismsdemandresponse}, whereby consumers choose their
consumption levels such as to minimize cost while respecting specific pre-determined
constraints (e.g., they run their HVAC such as to minimize operational costs while
keeping the temperature within pre-specified bands).
In earlier work, \citet{mehrabi2024optimalmechanismsdemandresponse} found that, under
their model, time-varying prices are achieve optimal demand response;\footnote{Specifically,
they show that---in a mean-field limit---prices can achieve demand-response benefits comparable
to what one could get under direct control, where a central planner selects each consumer's
consumption from within their indifference set.} and that optimal prices---in a time-homogeneous
setting---can be efficiently learned (i.e., with rates typical of first-order algorithms)
using only access to consumers' aggregated consumption
profiles at deployed prices. Here, we propose a neural network based algorithm that enables
us to use rich contextual information to learn day-specific prices, such as to accommodate
variable system conditions that are ubiquitous in practice. Specifically, we show the following: 

\begin{enumerate}
\item A first challenge is that the indifference-set approach asks us to learn price vectors constrained within a convex set.
We start by reviewing candidate algorithms for constrained-output learning, including projection-based methods, penalty-based relaxations, and differentiable optimization layers \citep{Bertsekas1997,Amos2017}. Essentially, these approaches either repeatedly project predicted prices back onto the feasible set, which can be unstable (with no guarantee that such methods converge); penalize constraint violations, which does not guarantee output feasibility; or embed an optimization problem inside the learning loop to always ensure that the output belongs to the feasible set, which can be computationally daunting. As such, we find existing methods for constrained-output learning not to be directly deployable in our setting.

  \item We introduce a cluster-based neural network architecture that learns context-dependent weights over a collection of feasible prototype prices. This design offers two major advantages: 1) convex combinations of feasible prototypes guarantee feasible prices, and 2) learning the optimal price requires only simple first-order updates, using standard back-propagation for the network weights and projected gradient steps for the prototype prices.

    \item  We evaluate our proposed contextual pricing method for demand management in realistic grid simulation environments powered by OpenDSS.\footnote{OpenDSS is an open-source distribution system simulator used to model power flows, loads, distributed generation, and operational constraints in electric networks.} Our experiments span multiple cities with heterogeneous grid and consumer conditions, including solar availability, demand levels, load composition, and consumer access to smart energy management devices. The results show that our contextual pricing framework delivers substantial demand-management gains over flat-rate pricing and vanilla context-agnostic dynamic pricing benchmarks. In particular, we show that the benefits of contextual pricing are more pronounced in solar-rich regions with high variability, where our proposed demand response solution can more effectively reveal the economic and environmental value of renewable integration.

\item We provide theoretical guarantees for the proposed architecture of contextual price learning. First, we show the \emph{universal approximation} property of our architecture. More precisely, we show that any contextual mapping with compact output constraints can be approximated arbitrarily well by the class of neural networks using sufficiently many clusters and hidden neurons. In addition, we analyze the \emph{learning dynamics} of our proposed architecture and show that the first-order updates (standard back-propagation for the network and projected gradient descent for the prototypes) still converge to a first-order stationary point of an underlying clustering problem, in which prices are clustered into groups without contexts (the vanilla context-free clustering problem).
\end{enumerate}

The rest of the paper is organized as follows. \cref{sec:conceptual} formulates contextual optimal pricing as a constrained learning problem and derives the optimal pricing representation. It also presents an end-to-end learning framework from observing exogenous conditions to updating context-dependent prices. \cref{sec:simulator} describes the household energy management system, device-level demand models, and grid simulator used to evaluate the proposed framework. \cref{sec:algeval} presents two sets of controlled simulation experiments.

The first set of experiments evaluates the benefits of contextual pricing deployed in settings with alternating weather conditions. We then compare our method with a pricing benchmark that is context-free and purely optimizes based on average system conditions. We then, in the second set of experiments, compare four contextual learning algorithms under a similar grid risk objective to assess their relative performance. In addition, \cref{sec:case} presents a case study on peak demand shaving across San Francisco, Atlanta, and Phoenix, illustrating how contextual pricing adapts to heterogeneous weather conditions and improves demand-response performance. We later, in \cref{sec:approx}, establish theoretical properties of the proposed architecture. More specifically, we first show the universal approximation guarantee for its expressive power and then present a convergence analysis of its learning dynamics to stationary points of an underlying clustering problem with no contexts.

\subsection{Related work} \label{sec:relatedwork}

Our work fits broadly withing the literature on demand response with time-varying prices, where we operate under the indifference-set approach of \cite{mehrabi2024optimalmechanismsdemandresponse}. This can be located more generally within the literature of works studying demand response as a mechanism for coordinating electricity consumption through price incentives \citep{peura2015dynamic,crew1995theory}. These mechanisms are often formulated as Stackelberg games, in which the grid operator acts as a leader setting prices while consumers respond optimally (minimizing their costs) by adjusting their consumption \citep{adelman2019dynamic,zhou2017incentive,zhou2019online}.

Regarding contextual pricing, there are two main research streams for demand management with side information. The first class of works incorporates exogenous system information (such as weather conditions, renewable generation forecasts, and demand forecasts) into electricity pricing models. Much of this literature focuses on price forecasting in wholesale electricity markets \citep{weron2014electricity,Morales2014}. In these models, the final goal is purely predictive rather than incentive design, which we focus on here for demand management.
The second stream of works incorporates exogenous system information directly into mechanism design, where the grid operator actively sets prices conditional on observed system conditions \citep{li2025day,feng2020stackelberg}. This can be inefficient, as only a small number of major states are identified, and the approach suffers from the lack of a smooth mapping from contexts to optimal prices. This can be seen as a coarse partitioning of the input space (fitting a piecewise constant function) to approximate the optimal mapping.
In our work, we aim to develop a learning-based pricing mechanism that learns a direct mapping from context to optimal incentives, adapting smoothly to exogenous conditions for optimal demand management.

In practice, many learning problems must output values that satisfy certain hard constraints \citep{Amos2017,agrawal2019differentiable}. As an example, we can consider the multi-quantile prediction problem where, conditional on covariates, outcome quantiles should be non-decreasing in the quantile level. If this prediction problem is solved without regard to these outcome conditions, then the estimated quantile values can cross and therefore violate the defining structure of a quantile function \citep{KoenkerBassett1978,ChernozhukovFernandezValGalichon2010}, making the result unacceptable for any meaningful use.
Existing solutions for machine learning with hard constraints on the outputs often fall into two main categories. One approach imposes constraints or corrections directly on the learned outputs, for example through post-hoc projection onto the set, or penalties to avoid falling outside feasible sets (e.g., discouraging quantile crossing \citep{ChernozhukovFernandezValGalichon2010,MoonJeonLeeKim2021,ShenJiaoLinHorowitzHuang2024}). The other approach designs the learning architecture so that feasibility is guaranteed by construction (e.g., parameterizing the output through monotone transformations or nonnegative increments \citep{Cannon2018}).
In our setting of contextual pricing, the hard constraint is not monotonicity as in multi-quantile prediction, but the feasibility of the price profile. Specifically, for every context $x$, the learned price must lie in the price set $\cP\subset \reals^d$, where this set is induced by the choice of grid risk function. In this problem, the first category of solutions for hard constraints, as discussed above, is not ideal, as it can suffer from instability in convergence, and smooth penalization can still output prices outside the feasible price set, violating the grid risk formulation and therefore leading to suboptimal demand management. We therefore focus on the second category and try to ensure that by construction the output is always in the feasible price set, while it enjoys desirable convergence guarantees.

Finally, our paper fits within a literature on deep learning algorithms---including
deep reinforcement learning---for learning optimal energy prices. Most existing contributions,
however, require access to much richer data or system information than we do.
Some existing methods assume that we are able to simulate realized consumption at
different counterfactual prices \citep{lu2018dynamic}, that we are able to learn
detailed consumer behavior models, e.g., by accessing building characteristics to predict
dynamics of HVAC operation \citep{yoon2020retail}, or that we have access to sufficient
price exploration to learn aggregated consumer demand functions in a
data-driven way \citep{xu2019deep}. \citet{fraija2024deep} consider a
coordinated model where a planner and households engage in iterated communication to
converge on prices. Finally, \citet{arlt2026deep} takes a more black-box approach, but
asks us to compare observed demand under time-varying prices to imputed demand under
fixed prices. Our approach does not require internal knowledge of household energy use;
does not rely on households engaging in repeated communication (which could potentially
be gameable); and does not require us to impute counterfactual demand. Rather, we are
able to learn simply from repeated observation of realized demand at deployed prices.

\section{Pricing for optimal demand response }\label{sec:conceptual}

We begin with the indifference-set view of demand response, which treats flexible electricity consumption as a choice among service-equivalent load profiles rather than as a reduction in the consumer's underlying energy needs. 
In this model, consumers are assumed to have high-level energy goals, such as keeping indoor temperature within a desired range or charging an EV within certain hours. Over a time horizon of $d \in \mathbb{N}$ periods (in our setting, for example, $d=24$), an \textit{indifference set} $\mathbf{R}\subset \reals^d$ represents the set of all consumption profiles that satisfy these goals. Thus, if $q\in \mathbf{R}$, then following the energy schedule $q$, that is, consuming $q_j$ units of energy in hour $j$, satisfies the consumer's requirements. 

One can view $q$ as the household's aggregate electricity consumption across all controllable devices. Since there may be many consumption profiles that satisfy the same high-level goals, the consumer is indifferent among them. For example, a household may consume more electricity at 2pm to pre-cool the home instead of turning on the cooling system after arriving home at 6pm. Indifference sets provide a formal way to capture this temporal flexibility: consumers can shift demand across time while preserving the same level of comfort or service.

In the indifference-set approach, pricing is modeled as a Stackelberg game between the grid operator, who sets the price, and households, who respond to the announced price.\footnote{The underlying assumption is that consumers are price-takers and have negligible individual incentives to behave strategically.} More specifically, the grid operator first announces a price schedule over the next $d$ periods. Households then react to this price by choosing consumption schedules that satisfy their energy goals while minimizing their total energy cost.

Formally, let $p\in\mathbb{R}^d$ denote the hourly electricity price, where $p_j$ is the price per unit of consumption in hour $j\in[d]$. Let $q_i\in\mathbb{R}^d$ denote the hourly consumption profile of household $i$, where $q_{ij}$ is the consumption of household $i$ in hour $j$. In our setting, we take $d=24$, corresponding to hourly price and consumption schedules. The two sides of the game solve the following problems.

\begin{enumerate}[label=\roman*)]

\item \textbf{Consumers.} In response to a price $p\in \reals^d$, each household, equipped with a home energy management system (HEMS), minimizes its total energy cost while satisfying its energy goals, encoded by its indifference set:
\begin{equation}
\label{eq:HEMS}
    q_i(p) \in \arg\min_{q_i\in \reals^d} p\cdot q_i\,, 
    \qquad \st \qquad q_i\in \mathbf{R}_i .
\end{equation}
Here, $\bR_i\subset \reals^d$ is the household's indifference set. It implies that the household's is indifferent within this set as long as their energy goals are satisfied. It is easy to observe that the selected consumption profiles are scale-invariant with respect to the price level: For any $\alpha>0$, the solution set of \eqref{eq:HEMS} is unchanged when the price vector $p$ is replaced by $\alpha p$. This follows because the indifference-set formulation targets demand \textit{substitution} over time rather than demand \textit{suppression}; hence, prices affect the timing of consumption, but not the underlying energy goals encoded by the indifference set.

\item \textbf{Grid operator.} On the other side of the game, the grid operator aims to manage risks induced by the aggregate net-demand curve. Given available renewable generation $Q^0\in \reals^d$ and aggregate demand $Q\in \reals^d$, the \textit{net demand} is $Q-Q^0$. Certain features of this curve are operationally important for the grid, such as its peak value $\|Q-Q^0\|_{\infty}$ or its total net demand $\|Q-Q^0\|_1$. More generally, for a grid risk function $\rho:\reals^d\to \reals_{\ge 0}$, and $n$ consumers consumption profile $q_i(p)$ collected under price $p$ the grid operator focuses on controlling
\begin{equation}
\label{eq:grid}
  p^*\in \arg \min_{p\in \reals^d} \;
    \rho\Big(\sum_{i} q_i(p)- Q^0\Big).
\end{equation}
\end{enumerate}
In this setup, we focus on risk functions $\rho(\cdot)$ that are convex and positively homogeneous of degree one; that is, $\rho(\alpha q)=\alpha \rho(q)$ for all $\alpha\ge 0$ and $q\in\reals^d$. A natural class of such risks is given by norms applied to linear transformations of the net-demand vector, or convex combinations of such terms. For example, for two linear maps $A$ and $B$ and any $\lambda\in[0,1]$, the function $\rho(q)=\lambda \|Aq\|+(1-\lambda)\|Bq\|$
is a valid risk function.

It is worth to mention that although this class of grid risk functions may appear simple, it is expressive enough to capture many operationally relevant demand management objectives, e.g., total net demand by $\ell_1$ norm and peak demand by the $\ell_\infty$ norm.

The pricing formulation in \eqref{eq:grid} is bilevel optimization problem. Mopre specifically, the grid operator chooses prices, but prices affect the grid objective only through the induced household responses $q_i(p)$. As a result, the objective need not be convex in $p$, and directly optimizing over price schedules can be a challenging optimization problem. The key insight of \citet{mehrabi2024optimalmechanismsdemandresponse} is that this interaction admits a simpler characterization in the mean-field limit, where the number of consumers grows to infinity. The full statement of this result is deferred to Appendix~\ref{sec:grad}, and here we use the main implication of their result.

Specifically, let $h_i(p):=p \cdot q_i(p)$ denote household $i$'s minimized energy cost under price $p$, with $q_i(p)$ denoting a corresponding optimal consumption profile as in \eqref{eq:HEMS}. We define the limiting average consumer cost as $H(p):=\E[h_i(p)]$, where the expectation is taken with respect to the randomness in the indifference set $\bR_i$ in \eqref{eq:HEMS}. In addition, let $q_0=\lim_{n\to\infty}Q^0/n$ denote the per-consumer renewable supply. Finally, let $\mathcal P$ be the dual set induced by the grid risk function $\rho$. More specifically, since $\rho(\cdot)$ is convex and positively homogeneous of degree one, it can be represented as the support function of a convex set $\mathcal P$; equivalently, the relation between grid risk function $\rho$ and the set of feasible prices $\cP$ is given as $\rho(q)=\sup_{p\in\mathcal P}p\cdot q$ \footnote{More precisely, for a convex, positively homogeneous risk function $\rho(\cdot)$, the associated dual set is
$\mathcal{P}=\{z\in\mathbb{R}^d: z\cdot q\le \rho(q), \ \forall q\in\mathbb{R}^d\}$, so that $\rho$ is the support function of $\mathcal{P}$ \cite[Claim A.2.24]{artstein2015asymptotic}. When $\rho$ is a norm, with dual norm $\rho^*(\cdot)$, this set admits the dual-norm representation
$\mathcal{P}=\{z\in\mathbb{R}^d:\rho^*(z)\le 1\}$.}.
Then the target price can be characterized as \citep{mehrabi2024optimalmechanismsdemandresponse}:
\begin{equation}
\label{eq:mean-field-price-characterization}
p^* \in \arg\max_{p\in\mathcal P}
\big\{
H(p)-q_0^\sT p
\big\}.
\end{equation}
This reformulation turns the original bilevel problem into a maximization problem over price profiles induced by the grid-risk criterion. Moreover, $H(\cdot)$ is concave in the price vector $p$, since each $h_i(\cdot)$ is the pointwise minimum of concave functions of $p$, and averages preserve concavity. Thus, the operator can characterize the optimal price through a concave programming problem.

\subsection{Contextual pricing}

For optimal demand response with pricing under heterogeneous system conditions, the optimal price characterization in \eqref{eq:mean-field-price-characterization} reads as follows. More precisely, for a context $X\in\mathcal X$ representing exogenous system conditions, the optimal price is given by
\begin{equation}\label{eq:p_star}
p^\star(X)\in \arg\max_{p\in \cP}
\E\big[ h_i(p)-q_0\cdot p \mid X\big]\,.
\end{equation}

Here, the conditional price characterization captures how the context $X$ affects the randomness in two components: i) the per-consumer scaled renewable supply $q_0$, since, for example, cloudy and sunny days induce different renewable-generation distributions; and ii) the consumer cost function $h_i(p)$ as household $i$'s minimized energy cost under price $p$,
through the context-dependent distribution of the indifference set $\mathbf{R}_i$. For instance, on a hot day, the set of consumption profiles that maintains the same indoor-temperature comfort can differ substantially from that on a cooler day. Thus, the distribution of feasible household schedules, and hence the induced cost function $h_i(.)$, changes with the daily context.

We focus on learning the mapping $p^*: \mathcal{X}\to \cP$, from each day's system conditions to the feasible price set. To this end, we focus on a class of parametrized learning models that can 1) be expressive enough to capture the complex relationship between contexts and prices, and 2) always generate a feasible price. The second component is especially important in our setup, as the feasible price set $\cP$ is induced by the grid risk function $\rho(.)$ and imposes a hard constraint on the relative magnitude of price levels needed for effective demand management.

To accommodate these two requirements, we use a prototype-based class of pricing rules. More precisely, let $Z=[z_1\ \cdots\ z_K]\in\reals^{d\times K}$ denote a matrix whose columns are $K$ price prototypes, with $z_k\in\cP$ for each $k\in[K]$, and let $\pi_\theta:\mathcal X\to\Delta_{K-1}$ be a contextual weighting rule, where $\Delta_{K-1}:=\{u\in\reals_{\ge 0}^K:\sum_{k=1}^K u_k=1\}$. The induced price rule is $p_{\theta,Z}(X)=Z\pi_\theta(X)$.

Further details on this learning architecture, together with a discussion of alternative hard-constrained output methods, are provided in Appendix~\ref{appendix:algreview}.

The heterogeneous price-characterization problem in \eqref{eq:p_star} can therefore be lifted to an optimization problem over the network parameters and price prototypes, given as:

\begin{equation}\label{eq:learning}
\max_{\theta\in\Theta,\; Z\in\mathcal Z} \;
\E_X\Big[
    \E\big[
        h_i(Z\pi_\theta(X))-q_0 \cdot Z\pi_\theta(X)
        \mid X
    \big]
\Big],
\end{equation}
where $\mathcal Z:=\{[z_1\ \cdots\ z_K]: z_k\in\cP\}$ and $\Theta$ denotes the admissible parameter set.

Having characterized the objective for the contextual price rule $Z\pi_\theta(\cdot)$, we now turn to how this rule can be learned from operational data. The relevant feedback structure is naturally sequential. On each day, the operator observes the context, deploys a day-ahead price schedule, and then observes the realized consumption profiles at the end of the day. Since the household cost is given by $h_i(p)=p\cdot q_i(p)$, the realized objective value can be precisely calculated from household-level consumption observations via smart meters. This makes it possible to learn the parameters $(\theta,Z)$ by sequential sample-average optimization.  

A naive approach would use zeroth-order methods, since the objective value can be evaluated after each deployed price. However, zeroth-order optimization can converge slowly, especially in high-dimensional parameter spaces and in the presence of noisy feedback \citep{DuchiJordanWainwrightWibisono2015}. The indifference-set formulation provides a computational advantage: it gives first-order information from the same observed consumption profiles. The next theorem shows that the gradient of the objective function in \eqref{eq:learning} can be written directly in terms of the realized aggregate demand response. This is useful both computationally and operationally. It enables first-order learning methods, while preserving privacy in the sense that the operator does not need to observe each household's preferences or indifference set $\bR_i$.

\begin{theorem}\label{thm:grad}
For a context $x\in \cX$, let $\mathsf J_\theta(x)\in\mathbb{R}^{m\times K}$ denote the transposed Jacobian of $\pi_\theta(x)$ with respect to $\theta$; that is, if $\theta\in\mathbb{R}^m$ and $\pi_\theta(x)\in\mathbb{R}^K$, then
\begin{equation}\label{eq:J}
[\mathsf J_\theta(x)]_{\ell,k}
:=
\frac{\partial [\pi_\theta(x)]_k}{\partial \theta_\ell},
\qquad \ell\in[m],\ k\in[K].
\end{equation}
Then, for the sample average objective in \eqref{eq:learning}, the gradient with respect to parameter $\theta$ is given by
\begin{equation}
\label{eq:chain-rule}
\nabla_{\theta}
\left[
    \frac{1}{n}\sum_{i=1}^n h_i(Z\pi_\theta(x))
    - q_{0}\cdot Z\pi_\theta(x)
\right]
=
\mathsf{J}_{\theta}(x)Z^\sT
\left(
    \frac{1}{n}\sum_{i=1}^n q_i(Z\pi_\theta(x))
    - q_{0}
\right).
\end{equation}
Here, $q_i(p)$ is the selected consumption by $\hems$ under price $p$ as given in \eqref{eq:HEMS}. 
\end{theorem}

The proof of Theorem~\ref{thm:grad} is deferred to Appendix~\ref{sec:grad}.

This theorem immediately motivates a sequential learning scheme for the contextual price rule. On each day, the operator observes a forecasted context $\widehat X_t$, forms a feasible price vector from the current prototypes as $p_t=Z_t\pi_{\theta_t}(\widehat X_t)$, and announces this price schedule day-ahead. Households then respond through their $\hems$ by selecting consumption profiles that minimize their energy cost within their indifference sets. At the end of each day, the gird operator observes the realized average consumption $\bar{q}_t$ and the realized per-consumer renewable supply $q_{0,t}$, and the difference $\bar{q}_t-q_{0,t}$ is the observed net demand value. By Theorem~\ref{thm:grad}, this value can be used  as a first-order information for updating the contextual weights $\theta$ and the price prototypes $Z$. Therefore, our algorithm  learns from operational feedback without observing household preferences or indifference sets directly.  The sequential structure of this feedback loop is illustrated in Figure~\ref{fig:sequential-learning}, and the detailed implementation steps are given in Algorithm~\ref{alg:day-ahead-contextual-price-learning}.

\begin{figure}[t]
\centering
\includegraphics[width=0.95\textwidth]{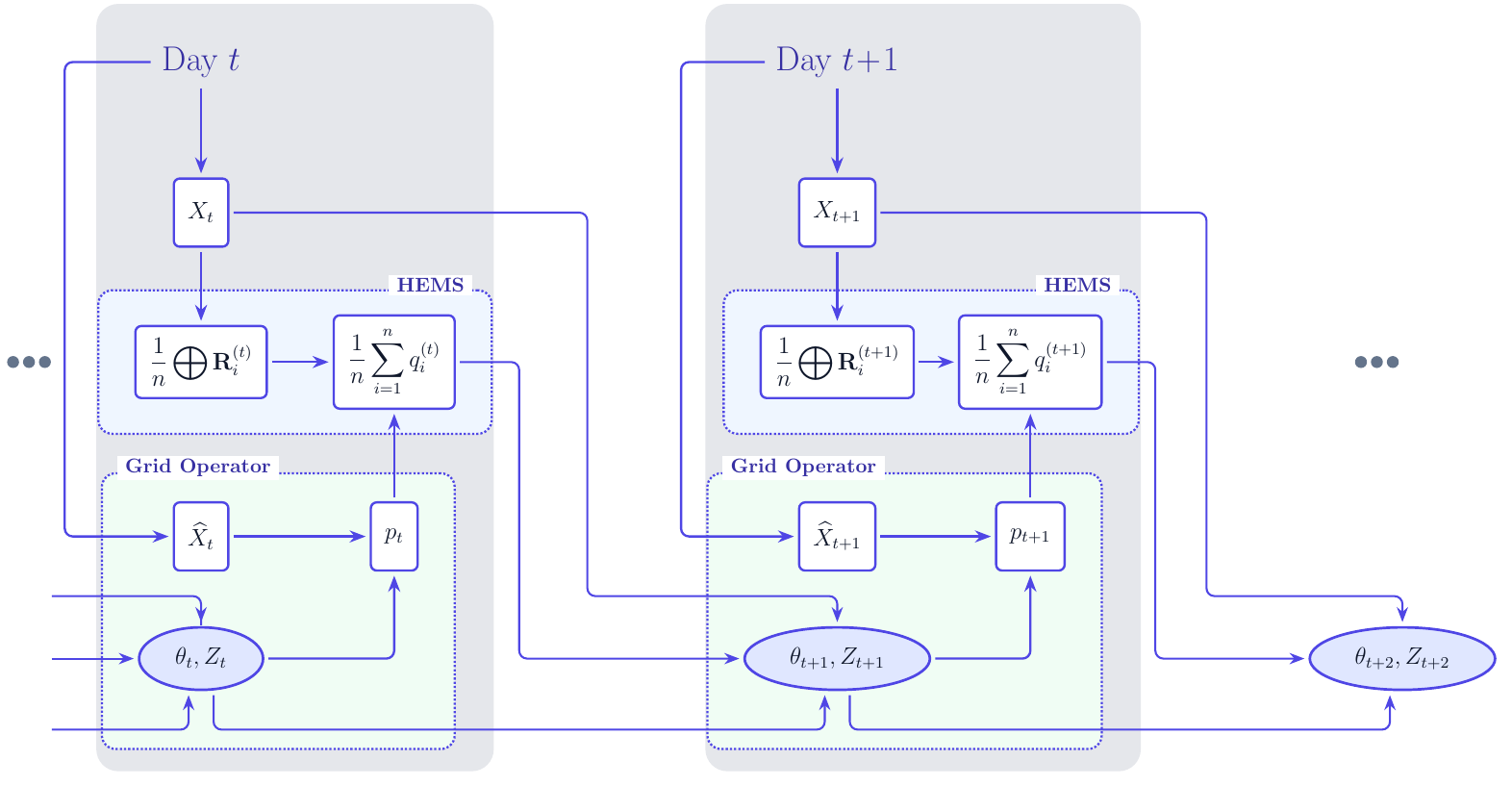}
\caption{Sequential contextual price learning. The operator announces the day-ahead price using the forecasted context and the current pricing rule. After operation, HEMS responses, realized consumption, and renewable supply reveal the actual net-demand signal, which is used to update the network weights and price prototypes. Better context forecasts lead to prices that are better aligned with realized system conditions and improve demand management.}
\label{fig:sequential-learning}
\end{figure}

\begin{algorithm}[t]
\caption{Day-Ahead Contextual Prototype-Based Price Learning}
\label{alg:day-ahead-contextual-price-learning}
\begin{algorithmic}[1]
\Require Number of periods $d\in \mathbb{N}$, number of households $n\in \mathbb{N}$, feasible price set $\cP\subset \reals^d$, number of price prototypes $K\in \mathbb{N}$, context space $\mathcal X$, step sizes $\eta_\theta,\eta_z>0$, initial network weights $\theta_0$, and initial price prototypes $\{z_{0,k}\}_{k=1}^K\subset \cP$.
\Statex \textbf{Output:} Price prototypes $Z_T=[z_{T,1}\ \cdots\ z_{T,K}]\in \reals^{d\times K}$ and contextual pricing rule $\pi_{\theta_T}:\mathcal X\to\Delta_{K-1}$, where $\Delta_{K-1}:=\{u\in\reals_{\ge 0}^K:\sum_{k=1}^K u_k=1\}$. 
\newline
For context $x$, price rule is $Z_T\pi_{\theta_T}(x)$. 
\Statex
\Statex \textbf{For} $t=0,1,\ldots,T-1$ \textbf{do}
    \Statex \hspace{\algorithmicindent}\textbf{Day-ahead pricing stage}
    \State \hspace{\algorithmicindent}Observe the forecasted context $\widehat X_t\in\mathcal X$.
    \State \hspace{\algorithmicindent}Stack the current prototypes columnwise as $Z_t=[z_{t,1}\ \cdots\ z_{t,K}]\in\reals^{d\times K}$.
    \State \hspace{\algorithmicindent}Compute the prototype weights $\pi_{\theta_t}$ and construct the price vector $p_t=Z_t\pi_{\theta_t}(\widehat X_t)$.
    \State \hspace{\algorithmicindent}Announce the price schedule $p_t$ for day $t$.
    \Statex
    \Statex \hspace{\algorithmicindent}\textbf{Operation stage}
    \State \hspace{\algorithmicindent} Each $\hems$ for $i\in [n]$ responds to $p_t$ via \eqref{eq:HEMS} and selects consumption profile $q_{i,t}(p_t)\in\reals^d$.
    \State \hspace{\algorithmicindent} The operator observes $\bar q_t:=\frac{1}{n}\sum\limits_{i=1}^n q_{i,t}(p_t)$ and the realized per-consumer renewable $q_{0,t}\in \reals^d$.
    \Statex
    \Statex \hspace{\algorithmicindent}\textbf{Learning stage}
\State \hspace{\algorithmicindent}Let $\mathsf{J}_{\theta_t}(\widehat X_t)$ denote the transposed Jacobian of the weight map $\pi_{\theta}(\widehat X_t)$ with respect to $\theta$ at $\theta=\theta_t$, as in \eqref{eq:J}. Update the network weights by
\[
    \theta_{t+1}
    =
    \theta_t
    +
    \eta_\theta\,
    \mathsf{J}_{\theta_t}(\widehat X_t)
    Z_t^\sT
    (\bar q_t-q_{0,t}).
\]
    \State \hspace{\algorithmicindent}Update each price prototype by
    \[
        z_{t+1,k}
        =
        \Pi_{\cP}
        \left(
            z_{t,k}
            +
            \eta_z\,[\pi_{\theta_t}(\widehat X_t)]_k(\bar q_t-q_{0,t})
        \right),
        \qquad k=1,\ldots,K.
    \]
\Statex \textbf{End for}
\Statex
\State \Return $Z_T$ and $\pi_{\theta_T}$, with final pricing rule $x\mapsto Z_T\pi_{\theta_T}(x)$.
\end{algorithmic}
\end{algorithm}

Algorithm~\ref{alg:day-ahead-contextual-price-learning} consists of several components that together define the end-to-end learning mechanism, which we describe in detail below. These components specify how system conditions are represented, how feasible price schedules are parameterized, how household responses are observed, and how the pricing rule is updated through operational feedback. Throughout the implementation, we take $d=24$, corresponding to hourly prices over a day.

\noindent\textbf{Covariates.}
The covariates $X_t$ used for pricing are available ahead of time, and cover features relevant to understanding both renewable production (e.g., solar radiation, cloud cover, length of sunlight, etc.), and consumer behavior (e.g., day of week effect, holidays). In our demonstration, we use the following covariates:
\begin{itemize}[itemsep=0pt, topsep=2pt]
\item Hourly temperature forecast ($^\circ$F).
\item Direct solar radiation forecast (W/m$^2$).
\end{itemize}

\noindent\textbf{Pricing.}
For network weights $\theta_t$ and prototype prices $Z_t=[z_{t,1}\ \cdots\ z_{t,K}]\in\mathbb{R}^{24\times K}$, with $z_{t,k}\in\cP$, the contextual weight map $\pi_{\theta_t}$ assigns weights in the simplex $\Delta_{K-1}$. Given the forecasted context $\widehat X_t$, the normalized demand-response price signal is $\widetilde p_t = Z_t\pi_{\theta_t}(\widehat X_t).$ When a retail-rate normalization or revenue target is required, the announced price can be written as
\[
p_t:=\lambda_t \widetilde p_t
=
\lambda_t Z_t\pi_{\theta_t}(\widehat X_t),
\qquad \lambda_t>0.
\]
Under the consumer response model in \eqref{eq:HEMS}, the multiplier $\lambda_t$ does not affect substitution behavior across time periods, since multiplying all hourly prices by the same positive scalar leaves the selected consumption profile unchanged. Thus, the learning problem focuses on the normalized price profile $\widetilde p_t\in\cP$, while $\lambda_t$ can be chosen separately to meet external regulatory or revenue objectives.

Without loss of generality, we assume that the price multiplier $\lambda_t = 1$ for all $t$, since a constant scaling factor does not affect the relative substitution behavior of consumption across time periods.\\

\noindent\textbf{Consumer Response.} We suppose  that each consumer $i$ operates a set of devices via a HEMS that monitors their individual needs. Specifically, each household may own devices associated with different energy needs and flexibility characteristics. Examples include:
\begin{itemize}
\item Inside lights: they are operated according to a pre-set schedule with no
flexibility, and contribute to the base load demand.
\item Heating, ventilation, and air conditioning (HVAC): they are used to maintain houses within target temperature bounds $T^{\min}_i, \,T^{\max}_i \in \mathbb{R}^{24}$.
\item Battery energy storage systems (BESS): they have an allowed state of charge (SOC) range of $[\underline{S}_i, \overline{S}_i]$ and a power range of $[\underline{pwr}_i, \overline{pwr}_i]$.
\end{itemize}

At a high level, the HEMS determines device-level operations in response to the posted prices $p_t$. We refer the reader to \cref{sec:HEMS} for a detailed mathematical formulation of individual devices; here, in \cref{HEMSop}, we outline the overall HEMS framework. Given the weather forecast, the HEMS first selects a target consumption level $q_{it}^{\mathsf{target}}\in \mathbb{R}^{24}$ that minimizes costs while satisfying pre-specified requirements. Then, it translates this target into a detailed operational plan, such as controlling the HVAC system to maintain indoor temperature trajectory at $T^{\mathsf{in}}_{it}\in\mathbb{R}^{24}$; or scheduling EV charging to reach a specified state-of-charge $x_{its}\in\mathbb{R}$ by time $s$, etc. Finally, the HEMS executes this plan, resulting in the realized hourly consumption $q_{it}\in \mathbb{R}^{24}$.

\begin{algorithm}[t]
\caption{HEMS Operation Framework}\label{HEMSop}
{\renewcommand{\baselinestretch}{1.4}\selectfont
\begin{algorithmic}
\State
\State[1] Pick a target consumption profile $q_{it}^{\mathsf{target}} = \arg\min_q\left\{q \cdot p_t : q \in \mathbf{R}_{it}\right\}\in \mathbb{R}^{24}$ that minimizes costs and meets pre-specified needs 

\State[2] Translate the chosen $q_{it}^{\mathsf{target}}$ into an operation plan during the day
\State[3] Follow the plan of operation, resulting in realized hourly consumption $q_{it}\in \mathbb{R}^{24}$
\State
\end{algorithmic}}
\end{algorithm}

We emphasize that following a structured operational plan, rather than mechanically adhering to a pre-specified consumption schedule, is critical for sustaining consumer welfare under our indifference set framework. Note that if the forecast inputs to the HEMS were perfectly accurate, then the planned consumption would match actual usage, i.e., $q_{it}^{\mathsf{target}} = q_{it}$. It is worth highlighting that any difference between forecast and realized conditions can result in a difference between planned and actual energy consumption. As an example, if the true outdoor temperature is higher than the forecast value, the computed operating schedule may no longer be optimal, resulting in higher cooling electricity consumption than what was originally expected. Such forecast errors introduce additional variability into aggregate load and may affect the performance of the learned pricing policy. \autoref{fig:forecast} illustrates the role of weather forecasts in the contextual learning framework and demonstrates how errors in the forecasting system may propagate through our pricing algorithm and demand management.

\begin{figure}[t]
    \centering
    \includegraphics[width=0.9\linewidth]{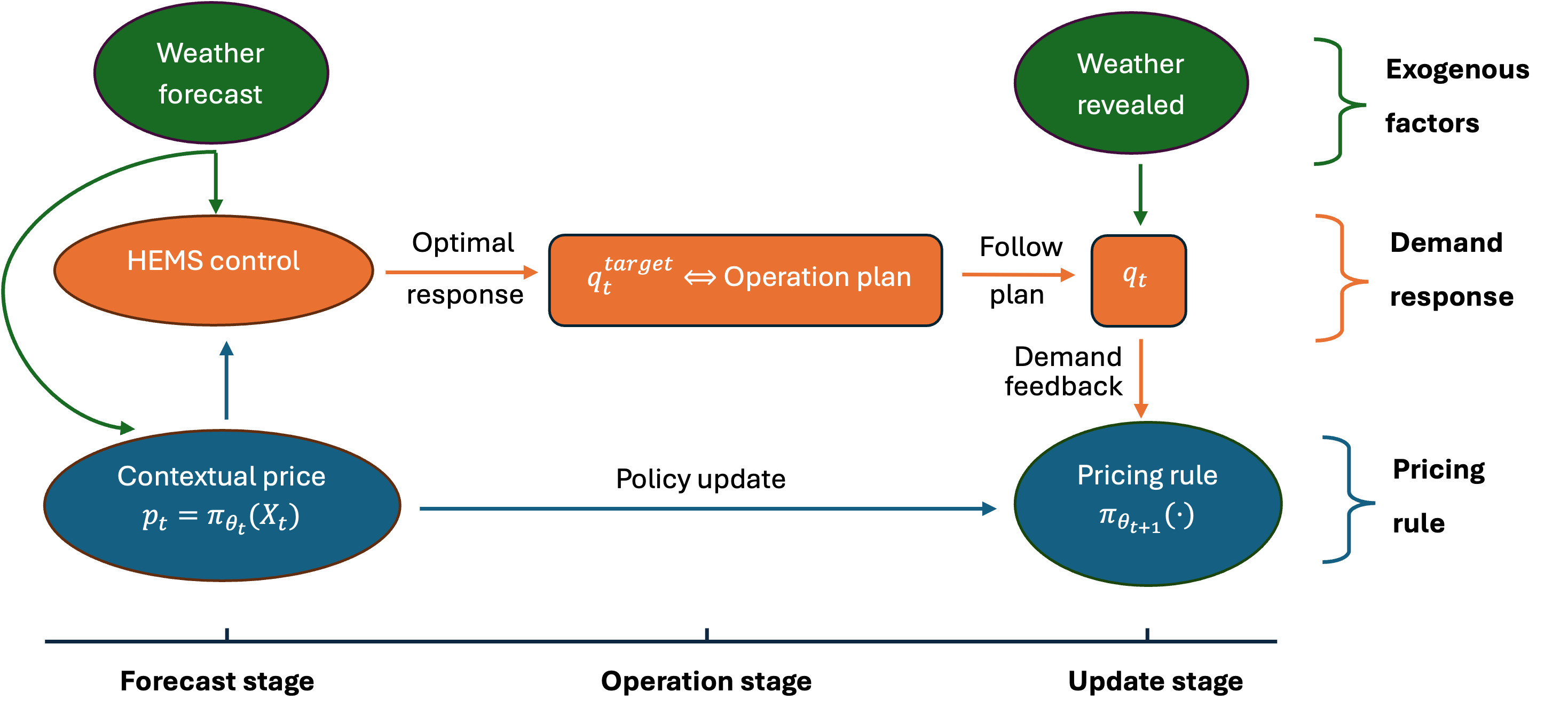}
    \caption{Role of weather forecasts in contextual price learning. Forecasted weather variables are used to generate context-dependent price signals prior to real-time operation. Households respond by optimizing device schedules (e.g., HVAC, BESS) based on the announced prices and anticipated conditions, yielding an operation plan. Forecast errors may lead to discrepancies between planned and realized consumption, as devices operate under realized conditions rather than the planned consumption trajectory. This introduces forecast-induced variability into demand feedback and subsequent price updates.}
    \label{fig:forecast}
\end{figure}

\noindent\textbf{Generation Cost.}\label{sec:cost}
The grid then observes the aggregate user-level demand $\sum_{i=1}^nq_{it}$, and renewable production $Q^0_t$. It must supply non-renewable electricity $\Delta_t = \sum_{i=1}^nq_{it} - Q^0_t \in \mathbb{R}^{24}$ to meet demand, incurring a generation cost $\rho(\Delta_t)$. In our demonstration, the required non-renewable electricity generation $\Delta_t$ will be obtained from grid simulation using OpenDSS \citep{opendss}. We refer the reader to \cref{sec:grid} for detailed power system settings.

\noindent\textbf{Learning.}
Lastly, we update the prototype-based pricing rule $p_t=Z_t\pi_{\theta_t}(X_t)$. Theorem~\ref{thm:grad} characterizes the corresponding first-order update for the contextual weights $\theta_t$, while the prototype updates are projected onto the feasible set $\cP$. Feasibility is enforced by construction: each prototype satisfies $z_{t,k}\in\cP$, and the network output satisfies $\pi_{\theta_t}(X_t)\in\Delta_{K-1}$, so the normalized price $\widetilde p_t=Z_t\pi_{\theta_t}(X_t)$ remains in $\cP$ whenever $\cP$ is convex.

\section{Simulator design}\label{sec:simulator}
In this section, we describe the implementation of the HEMS control and the power system simulator. We utilize a distribution network from the SMART-DS dataset \citep{krishnan2017smart}, which is representative of a typical North American suburban feeder. We model a total of 486 residential households, each equipped with a HEMS. The households are distributed to resemble a typical suburban neighborhood, enabling the simulation to capture both local constraints (e.g., within neighborhoods) and overall system behavior at the substation level. 

We begin with a detailed consumer-level device modeling and then present the grid-level generation cost formulation. For further details on our configuration of each device type, we refer readers to Section IV of \cite{zhou2026unlocking}.

\subsection{Household device modeling} \label{sec:HEMS}
On a given day $t$, the HEMS translates the user’s needs for each device into an indifference set $\mathbf{R}_{it} \subset \mathbb{R}^{24}$, which contains all consumption profiles that satisfy the demand. In particular, we model these sets as follows:

\vspace{0.2cm}
\noindent\textbf{Inside lights.} Due to inflexibility of lighting demand, the indifference set is a singleton $q_{it}^{\mathsf{LIGHTS}}\in \reals^{24}$, i.e.,
\begin{equation}\label{eq:lights}
\mathbf{R}_{it}^{\mathsf{LIGHTS}} = \big\{q_{it}^{\mathsf{LIGHTS}}\big\}.
\end{equation}

\noindent\textbf{Heating, ventilation, and air conditioning (HVAC).} The HEMS starts by querying an hourly forecast $T_{it}^{\mathsf{out}} \in \mathbf{R}^{24}$ for the outdoor temperature of day $t$. Inside temperature bounds are given as $T^{\min}_i, \,T^{\max}_i$. The HVAC can spend energy $q_{its}^{\mathsf{AC}}$ on cooling and $q_{its}^{\mathsf{H}}$ on heating at each hour $s$. Then the HEMS uses a linear heat transfer model with non-negative parameters $\alpha_i$, $\beta_i^{\mathsf{AC}}$, $\beta_i^{\mathsf{H}}$, to forecast inside temperature based on outside temperature and energy usage:
\begin{align*}\label{eq:linear-heat}
T^{\mathsf{in}}_{its} = T^{\mathsf{in}}_{it(s-1)} + \alpha_i (T^{\mathsf{out}}_{its} - T^{\mathsf{in}}_{it(s-1)}) + \beta_i^{\mathsf{H}} q_{its}^{\mathsf{H}} - \beta_i^{\mathsf{AC}} q_{its}^{\mathsf{AC}}.
\end{align*}
The resulting indifference set characterizes all consumption profiles that satisfy the temperature constraints under forecasted conditions, i.e.,
\begin{equation}\label{eq:HVAC}
\begin{aligned}
\mathbf{R}_{it}^{\mathsf{HVAC}} = \{q_{its}^{\mathsf{H}} + q_{its}^{\mathsf{AC}}:   &\forall 1\leq s\leq 24 ,\;T_i^{\mathsf{min}} \leq (1-\alpha)^s T_{it0}^{\mathsf{in}} \\ & + \sum_{k=1}^s (1-\alpha)^{s-k}(\alpha T_{itk}^{\mathsf{out}} + \beta_i^{\mathsf{H}}q_{itk}^{\mathsf{H}} - \beta_i^{\mathsf{AC}}q_{itk}^{\mathsf{AC}}) \leq T_i^{\mathsf{max}}\}.
\end{aligned}
\end{equation}

\noindent\textbf{Battery energy storage system (BESS).} The batteries consume $q_{it}^{\mathsf{BESS}} \in \mathbb{R}^{24}$ that modifies its state of charge $S_{it} \in \mathbb{R}^{24}$, which satisfies
\begin{equation}\label{eq:BESS}  
\begin{aligned}
\mathbf{R}_{it}^{\mathsf{BESS}} = \{q_{it}^{\mathsf{BESS}}:\; &S_{it} = \mathcal{C}_i \cdot q_{it}^{\mathsf{BESS}} + \bm{1}\cdot S_{it}^0,\\
\quad \; \, &\underline{S}_i \leq S_{it} \leq \overline{S}_i,\quad
\underline{pwr}_i \leq q_{it}^{\mathsf{BESS}} \leq \overline{pwr}_i\},
\end{aligned}
\end{equation}
    where $\mathcal{C}_i$ denotes the charging matrix that captures the impact of charging and discharging power on the SOC, and $S_{it}^0$ represents the initial SOC for household $i$ on day $t$.

The household-level devices deployed in the simulator consists of the above three categories. For $M$  total number of devices, the aggregate indifference set for user $i$ on day $t$ is defined as the Minkowski sum of the device-level feasible sets in \eqref{eq:lights}--\eqref{eq:BESS}, more precisely we have
\begin{equation}
\mathbf{R}_{it}:= \bigoplus\limits_{m = 1}^M \mathbf{R}_{it}^{(m)} = \Big\{\sum_{m=1}^M q_{it}^{(m)}: q_{it}^{(m)}\in \mathbf{R}_{it}^{(m)}\Big\}\,.
\end{equation}
\subsection{Grid generation cost}\label{sec:grid} We adopt a cost function that captures both the temporal variation and the aggregate level of the consumption profile. Specifically, for number of periods $d=24$ we define $D \in \mathbb{R}^{d \times d}$ as a Toeplitz matrix encoding first-order differences, such that $q^TDq = (q_{d}-q_1)^2 + \sum_{j=1}^{d-1} (q_{j+1}-q_j)^2$. Let $K := I + \lambda D^TD$. The cost function is then defined as
\begin{equation} \label{eq:gridcost}
    \rho(q) = (||q||_2^2 + \lambda q^TDq)^{1/2} :=||q||_K.
\end{equation}
Under the dual-norm formulation in \eqref{eq:mean-field-price-characterization}, the price set becomes
\begin{equation*}
\mathcal{P} =\{z \in \mathbb{R}^{d}:z^T K^{-1}z\leq1\}.
\end{equation*}
Projection onto this set can be formulated as the optimization problem
\begin{equation*}
    \arg\min_p ||z-p||_2\qquad s.t.\ p\in\mathcal{P}.
\end{equation*}
This problem admits the following characterization: if $z^TK^{-1}z \leq 1$, then $p^* = z$; otherwise, $p^* = (I+\mu K^{-1})^{-1}z$, where $\mu>0$ is chosen such that $p^{*T}K^{-1}p^{*} =1$. The form of the solution follows from the first-order optimality conditions of the associated Lagrangian, and $\mu$ can be efficiently determined via a bisection method \citep{boyd2004convex}. In our experiments, this grid-level optimization is solved at each simulation step, and the resulting contextual price signal is provided to the HEMS for household device scheduling.

\section{Algorithm evaluation}\label{sec:algeval}
In this section, we evaluate the proposed algorithm in \cref{subsec:soft} from two complementary perspectives. First, we illustrate the benefits of incorporating contextual information into the pricing signal, particularly in environments with significant day-to-day variation, by comparing it with a naive feedback-based price update rule. Second, we benchmark several contextual learning algorithms introduced in Appendix~\ref{appendix:algreview} and highlight the numerical advantages of the proposed cluster-based soft classifier approach.

For both the simulations in this section and the case study presented in \cref{sec:case}, we utilize historical weather data and historical forecast weather data obtained from Open-Meteo\footnote{Historical realized and forecast weather data are available at \url{https://open-meteo.com/}.}.

\subsection{Learning heterogeneity}
We begin by illustrating the advantages of contextual price learning in a heterogeneous environment. In particular, when system conditions vary significantly across days, a homogeneous pricing rule that is oblivious to the contextual information of each day will converge to a price signal adapted to the average system condition, which performs poorly under heterogeneous system conditions. Contextual pricing, in contrast, allows the algorithm to adapt prices to the latest system conditions. To this end, we construct a hypothetical setting in which weather conditions follow an alternating pattern:

\begin{equation*}
    X_t = \begin{cases}
        X_{t,\text{type I}}, &\text{if } t \text{ odd},\\
        X_{t,\text{type II}}, &\text{if } t \text{ even},
    \end{cases}
\end{equation*}
where type I conditions correspond to high temperature and low solar irradiance, while type II conditions correspond to low temperature and high solar irradiance. These two regimes imply different optimal pricing schedules. Under type I conditions, electricity demand (particularly electricity used for cooling) is high while renewable supply is limited. As a result, higher daytime prices are desirable to encourage demand shifting toward off-peak hours. Under type II conditions, electricity demand is relatively low while renewable generation is abundant. In this case, lower daytime prices help encourage electricity consumption so that more renewable energy could be consumed at low costs.

\autoref{fig:price_switch} shows the price signals generated by both the feedback-only algorithm and the contextual learning algorithm for each day type, after training over 90 days under the alternating weather pattern. For the cluster-based soft classifier algorithm, we set the number of clusters to $K=2$. The algorithm successfully captures the heterogeneity between the two day types and learns distinct pricing profiles for each regime. In contrast, the feedback-only algorithm assumes a homogeneous environment and therefore converges to a single price signal. As a result, the learned price profile is approximately the average of the two regime-specific price signals obtained from the contextual learning approach.

\begin{figure}[t]
    \centering
    \includegraphics[width=0.75\linewidth]{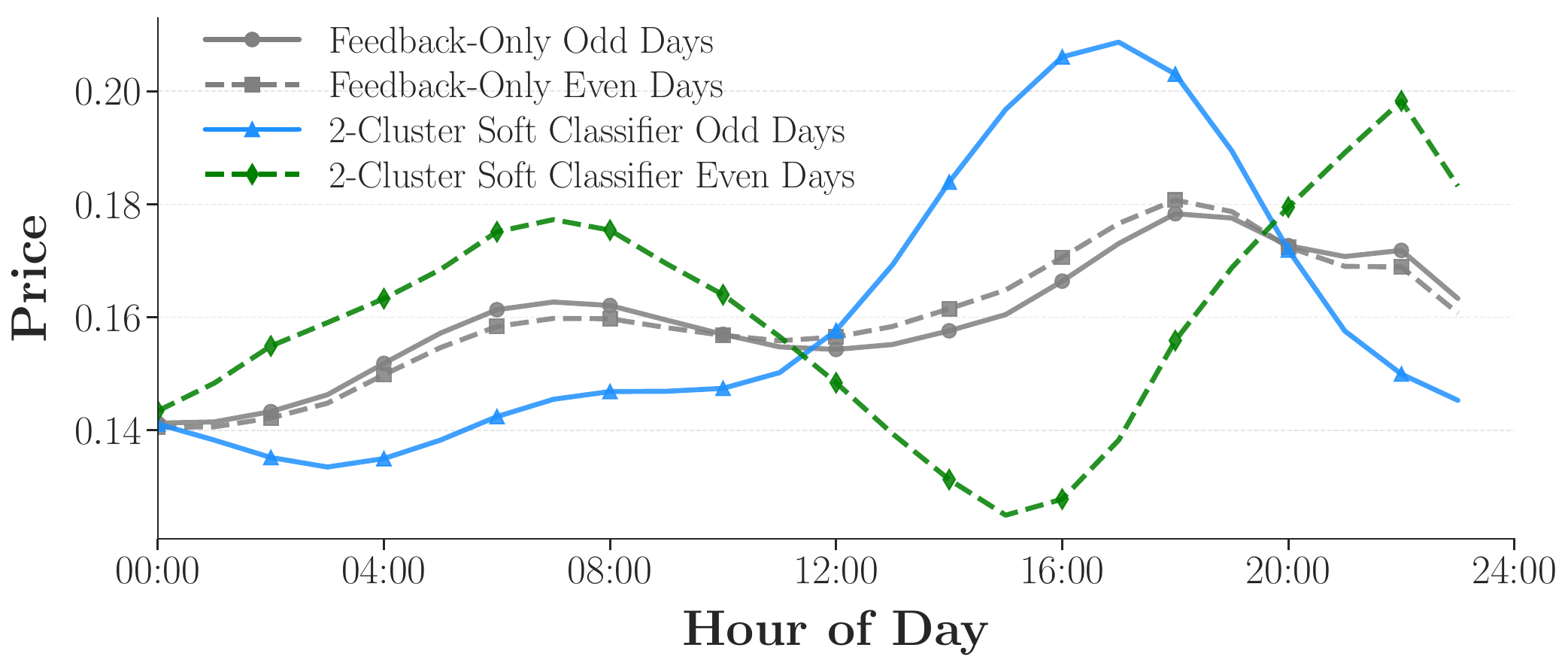}
    \caption{Comparison of price signals generated by the feedback-only and contextual learning algorithms under an alternating weather regime. The cluster-based soft classifier algorithm (with the number of clusters $K=2$) learns distinct price profiles corresponding to the two day types. In contrast, the feedback-only algorithm converges to a single price signal that is approximately the average of the two contextual prices.}
    \label{fig:price_switch}
\end{figure}

We also examine the resulting grid risk under the two pricing algorithms. As a key performance metric, we introduce peak demand shaving (PDS), defined as the percentage reduction in daily peak demand under the demand response program relative to a flat-rate policy, i.e.,
\begin{equation*}
    \text{PDS}= \frac{\max(\sum_{i=1}^{n}q_i^{\mathsf{flat}}-Q^0)-\max(\sum_{i=1}^{n}q_i^{\mathsf{DR}}-Q^0)}{\max(\sum_{i=1}^{n}q_i^{\mathsf{flat}}-Q^0)} \times 100\%.
\end{equation*}
Peak demand shaving is a widely recognized objective of demand response programs \citep{albadi2008summary}, as it mitigates grid stress and reduces system operating costs. \autoref{fig:PDS_switch} highlights the improvement in peak demand shaving achieved by incorporating contextual information into the pricing mechanism. The cluster-based approach consistently achieves higher reduction rate compared with the naive feedback-only algorithm. The intuition is straightforward: by utilizing forecast information, the contextual algorithm can quickly adapt to different day types by leveraging the learned cluster structure, effectively recalling regime-specific prices. In contrast, the feedback-only algorithm must gradually adjust its price signal through demand feedback, requiring several days to adapt. Consequently, it performs poorly in environments with substantial day-to-day variation.
\begin{figure}[t]
    \centering
    \includegraphics[width=0.75\linewidth]{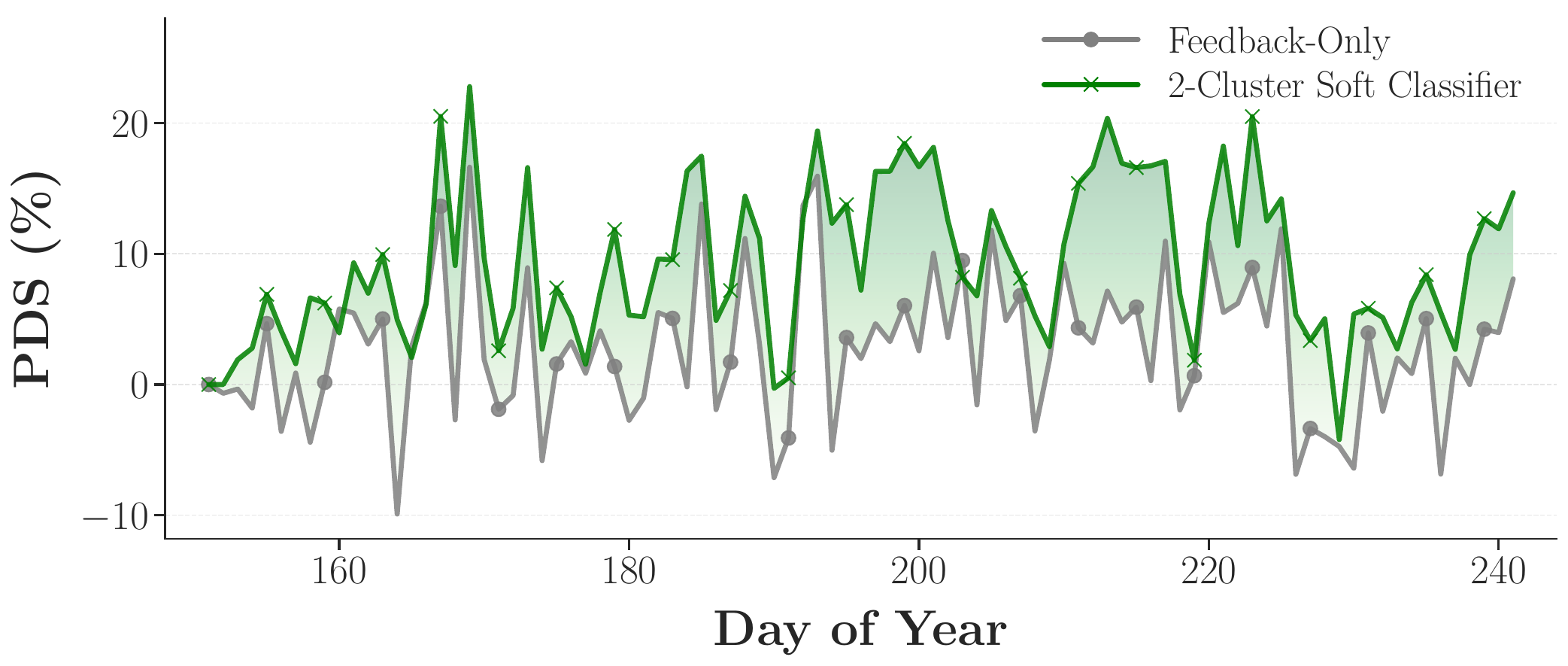}
    \caption{Comparison of peak demand shaving (PDS) achieved by the feedback-only and contextual pricing algorithms under an alternating weather regime. The shaded area represents the PDS gain achieved by contextual pricing over the feedback-only benchmark. The contextual pricing algorithm consistently attains greater peak demand reduction.}
    \label{fig:PDS_switch}
\end{figure} 

\subsection{Benchmark algorithms}
Next, we compare the performance of different contextual learning algorithms introduced in Appendix~\ref{appendix:algreview}. All four algorithms are evaluated under the same grid risk function, chosen as the $l_2$-norm, i.e. $\rho(z)=||z||_2$, since both the projection layer and interior point methods require differentiable constraints. All algorithms are trained over 90 days of weather data in Denver, CO, sampled from June to August 2022.

\autoref{fig:pds4meth} presents the peak demand shaving over the entire training period. Both the projection layer and interior point methods exhibit high volatility, occasionally producing negative peak demand shaving. This is consistent with the analyses in \cref{tab:summary}, where we emphasize that the projection layer method suffers from over-parameterization, while the interior point method exhibits instability near constraint boundaries. On the other hand, the proximal gradient method achieves the second best performance, despite relying on a weak dual characterization. Overall, the cluster-based soft classifier algorithm outperforms all other candidates, and is the only algorithm consistently achieving positive peak demand shaving.
\begin{figure}[t]
    \centering
    \includegraphics[width=0.75\linewidth]{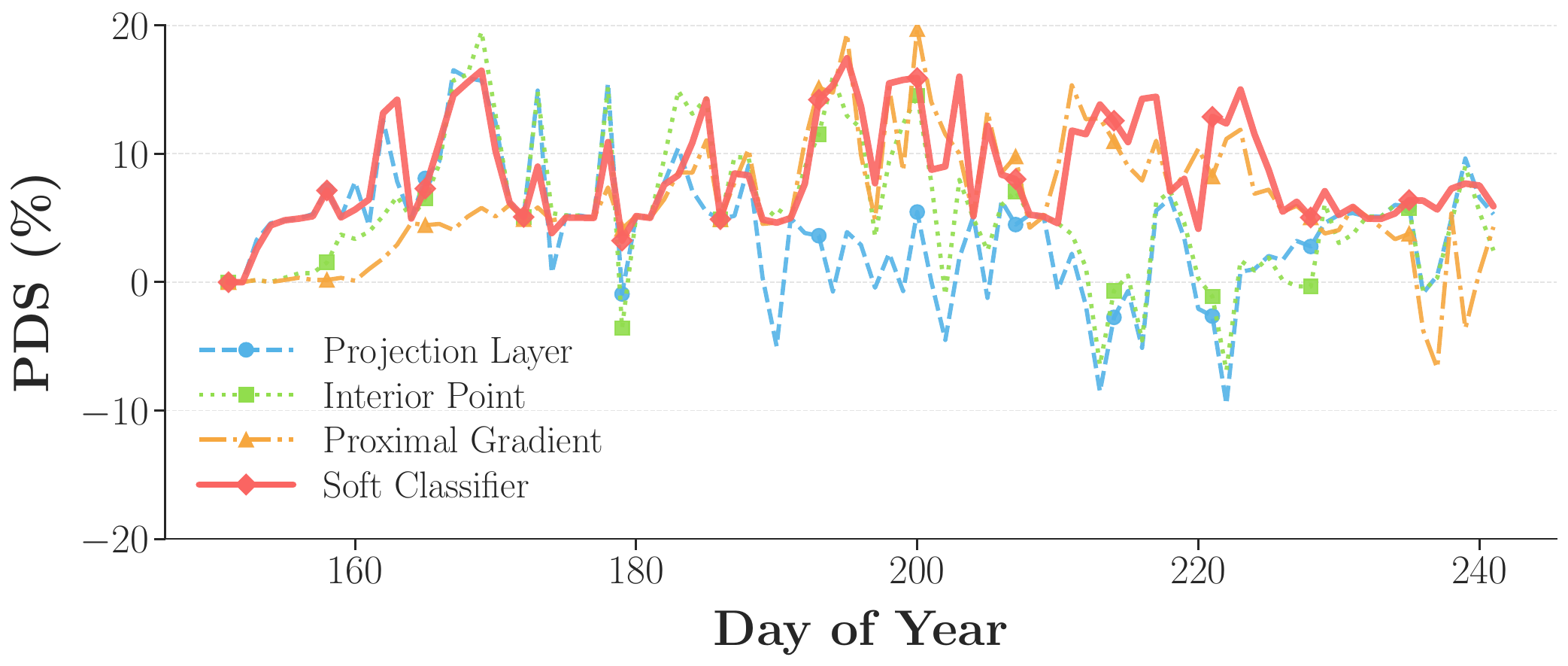}
    \caption{Comparison of peak demand shaving (PDS) across different contextual learning algorithms trained over 90 days of weather data in Denver, CO. Both the projection layer and interior point methods show high volatility, with occasional negative PDS values. The proximal gradient method performs moderately well, while the cluster-based soft classifier is the only algorithm that consistently achieves positive PDS.}
    \label{fig:pds4meth}
\end{figure}

\autoref{fig:PDScomp} provides a summarized comparison of the four algorithms in the form of a violin plot. The results indicate that the cluster-based soft classifier algorithm achieves the highest average reduction in peak demand while exhibiting the lowest variability. This suggests that the algorithm consistently delivers substantial peak reductions and avoids undesirable outcomes more effectively than the other candidates.
\begin{figure}[t]

    \centering
    \includegraphics[width=0.75\linewidth]{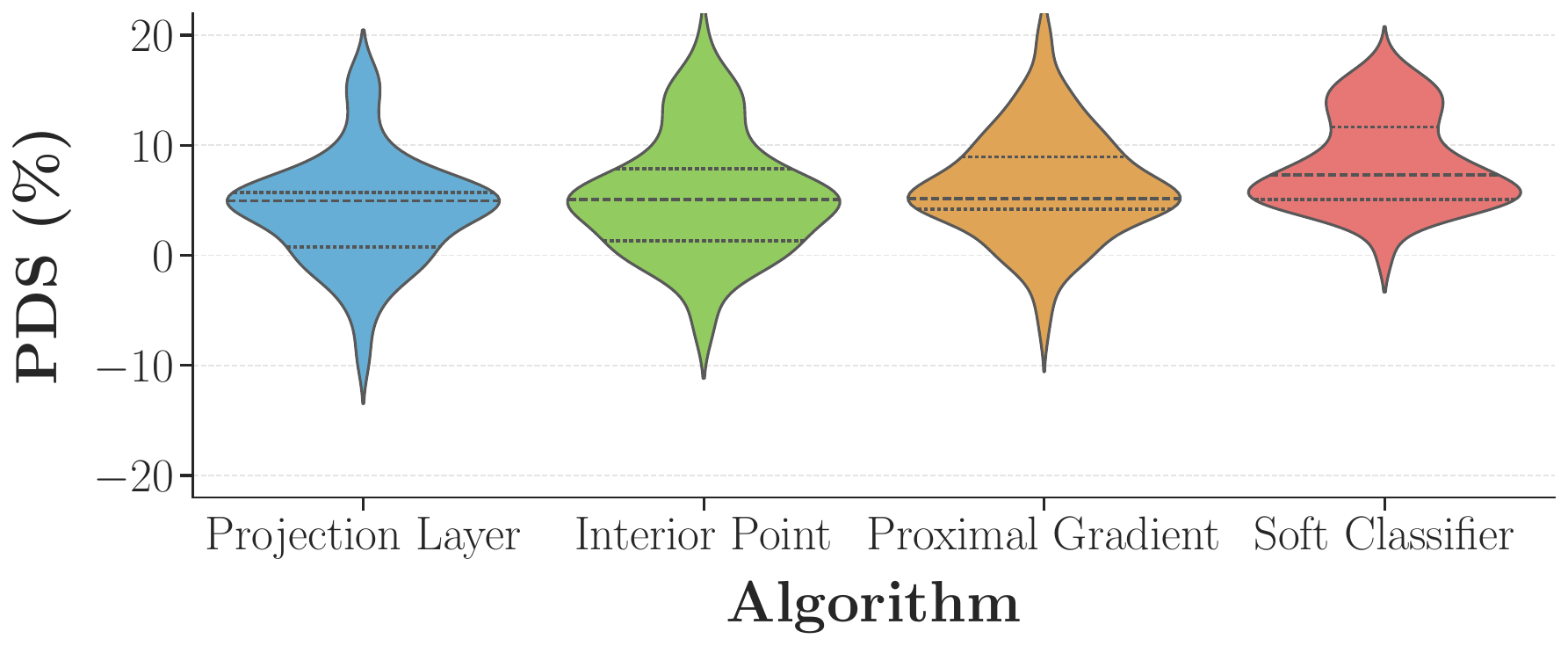}
    \label{fig:violin4meth}
    
    \caption{Violin plot comparing peak demand shaving (PDS) across the four contextual learning algorithms. The cluster-based soft classifier exhibits the highest average PDS with the smallest spread.}
    \label{fig:PDScomp}
\end{figure}

\section{Case study: value of demand response across cities}\label{sec:case}
Finally, we conduct a series of case studies across three U.S. cities with markedly different weather patterns, load characteristics, and renewable generation profiles. By conducting the analyses in realistic operating environments, we can examine the robustness of the learned pricing policies to exogenous variations, and highlight the practical benefits of context-aware pricing for large-scale electricity systems.

Specifically, we choose three cities showing significantly different climate conditions during the summer, as illustrated in \autoref{fig:weather3cities}. More precisely, 
\begin{itemize}
    \item San Francisco, CA. It is associated  with a Mediterranean climate with mild, relatively stable temperatures and moderate solar irradiance
    \item Atlanta, GA. It is associated with humid subtropical climate, with comparatively higher temperatures and lower effective solar irradiance due to cloud cover and rain
    \item Phoenix, AZ. It is associated with a hot desert climate, with persistently high temperatures and strong solar irradiance.
\end{itemize} 

These cities therefore span a diverse set of system conditions that directly affect both electricity demand (as it appears in cooling loads) and renewable generation. This diverse set of environments helps us show how heterogeneity in system conditions affects peak demand reduction. In addition, it helps us identify the system conditions under which contextual pricing yields the greatest operational benefits.

\begin{figure}[t]
    \centering
    
    \begin{subfigure}{0.48\linewidth}
        \centering
        \includegraphics[width=\linewidth]{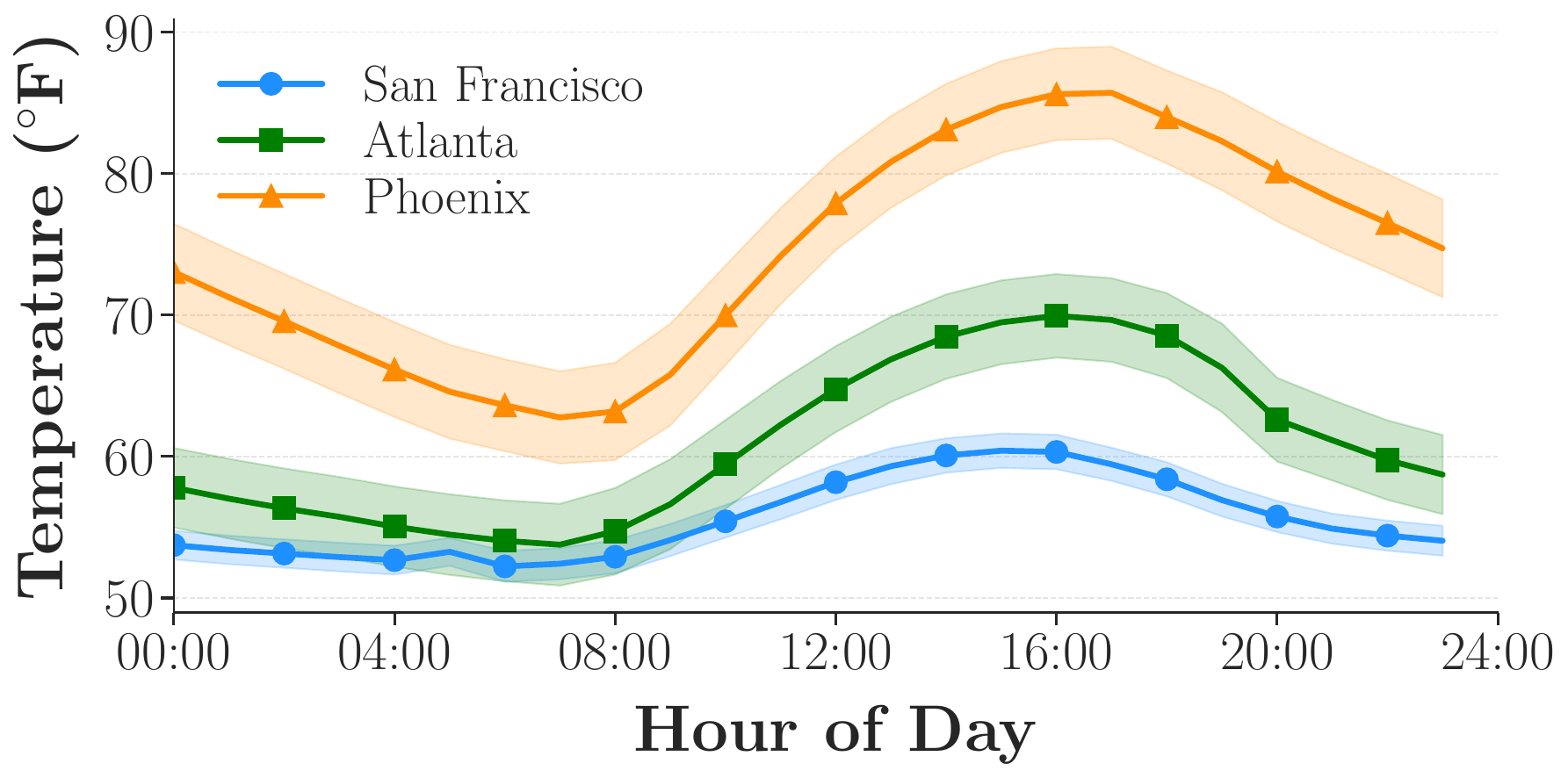}
        \caption{Temperature profiles}
        \label{fig:temp3cities}
    \end{subfigure}
    \hfill
    \begin{subfigure}{0.48\linewidth}
        \centering
        \includegraphics[width=\linewidth]{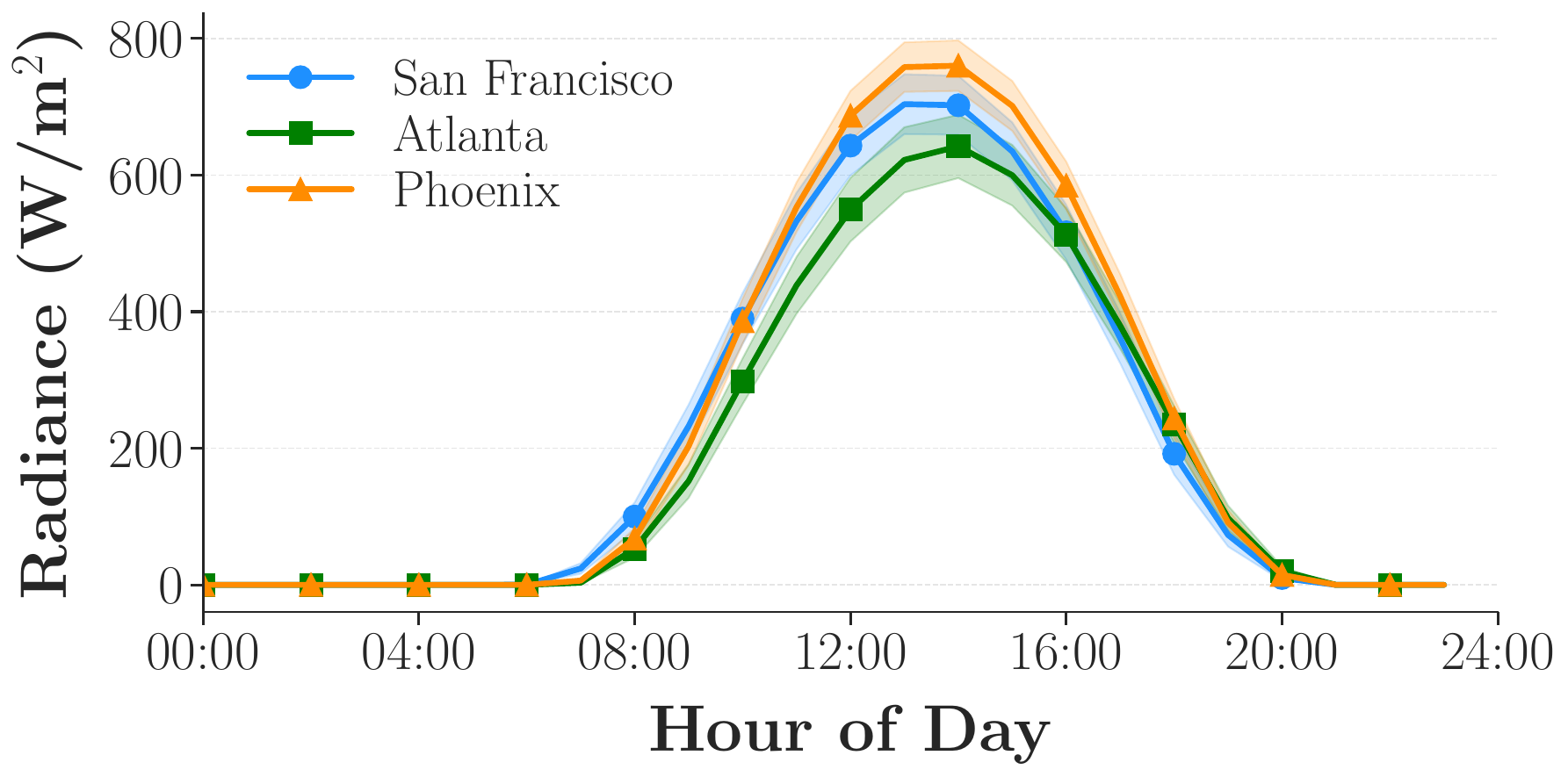}
        \caption{Solar irradiance profiles}
        \label{fig:rad3cities}
    \end{subfigure}
    
    \caption{Weather profiles of San Francisco, Atlanta, and Phoenix in July 2022. Panel (a) shows the daily temperature profiles; panel (b) shows the daily solar irradiance profiles. Solid lines represent the mean daily profiles across the month, while the shaded regions indicate the corresponding day-to-day variability ($\pm 1$ standard deviation).}
    \label{fig:weather3cities}
\end{figure}

\vspace{3mm}
\autoref{fig:PDS3cities} shows the peak demand shaving for each of the three cities. An extended case study analysis is provided in Appendix \ref{appendix:case}. We assume that all households are equipped with home energy management systems (HEMS) that coordinate indoor lighting, HVAC operation, and battery storage.

First, peak demand shaving is positive on the vast majority of days across all three cities, with only a small number of exceptions. This indicates that the proposed learning framework consistently produces price signals that outperform the flat-rate benchmark in mitigating peak demand. In other words, even without perfect forecast, the learned contextual pricing rule systematically improves load shaping relative to static pricing.

Second, we observe substantial heterogeneity in the magnitude of peak reduction across cities. The algorithm achieves the highest average peak demand shaving in Atlanta, while the lowest average reduction is observed in San Francisco. These differences can be understood through the lens of demand flexibility, which plays a central role in the indifference-set-based framework. The effectiveness of contextual pricing depends critically on the extent of flexible load that can respond to time-varying price signals. In our setting, HVAC is the primary device whose usage is strongly influenced by exogenous factors. When cooling needs are either minimal or effectively inflexible, the effect of demand substitution diminishes. 

This is reflected in the outcomes for San Francisco and Phoenix. In San Francisco, relatively mild summer temperatures lead to limited air-conditioning demand, reducing the available flexible load to shift. In Phoenix, although cooling demand is substantial, extreme temperatures constrain flexibility because reducing HVAC operation would violate comfort-level requirements. In contrast, the results for Atlanta suggest that the proposed mechanism performs particularly well in environments with moderate but variable cooling demand, leading to sufficient flexibility for demand substitution without compromising comfort constraints. This cross-city variation reveals that performance is not uniform across environments but instead depends on underlying climate-related factors that shape demand flexibility as well as renewable generation.

\begin{figure}[t]
    \centering

    \begin{subfigure}{0.32\linewidth}
        \centering
        \includegraphics[width=\linewidth]{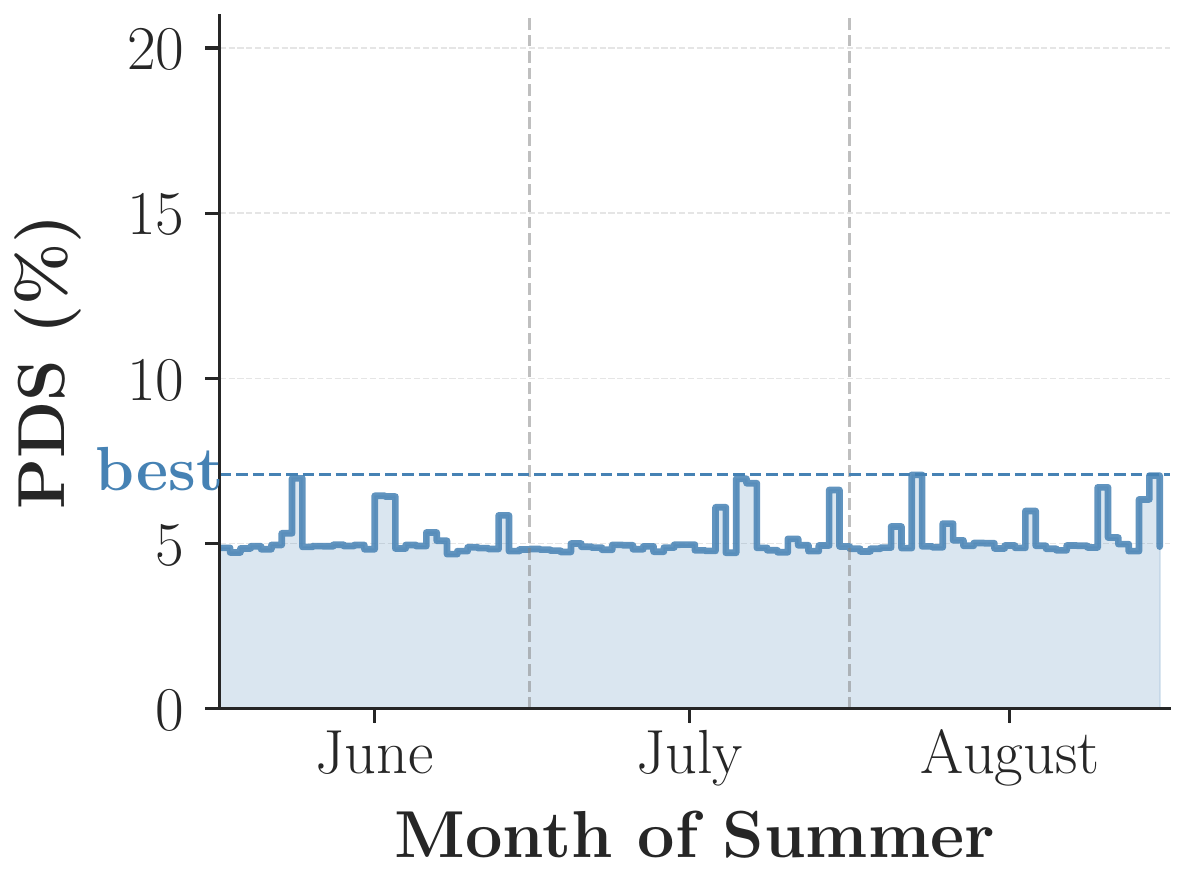}
        \caption{San Francisco}
    \end{subfigure}
    \hfill
    \begin{subfigure}{0.32\linewidth}
        \centering
        \includegraphics[width=\linewidth]{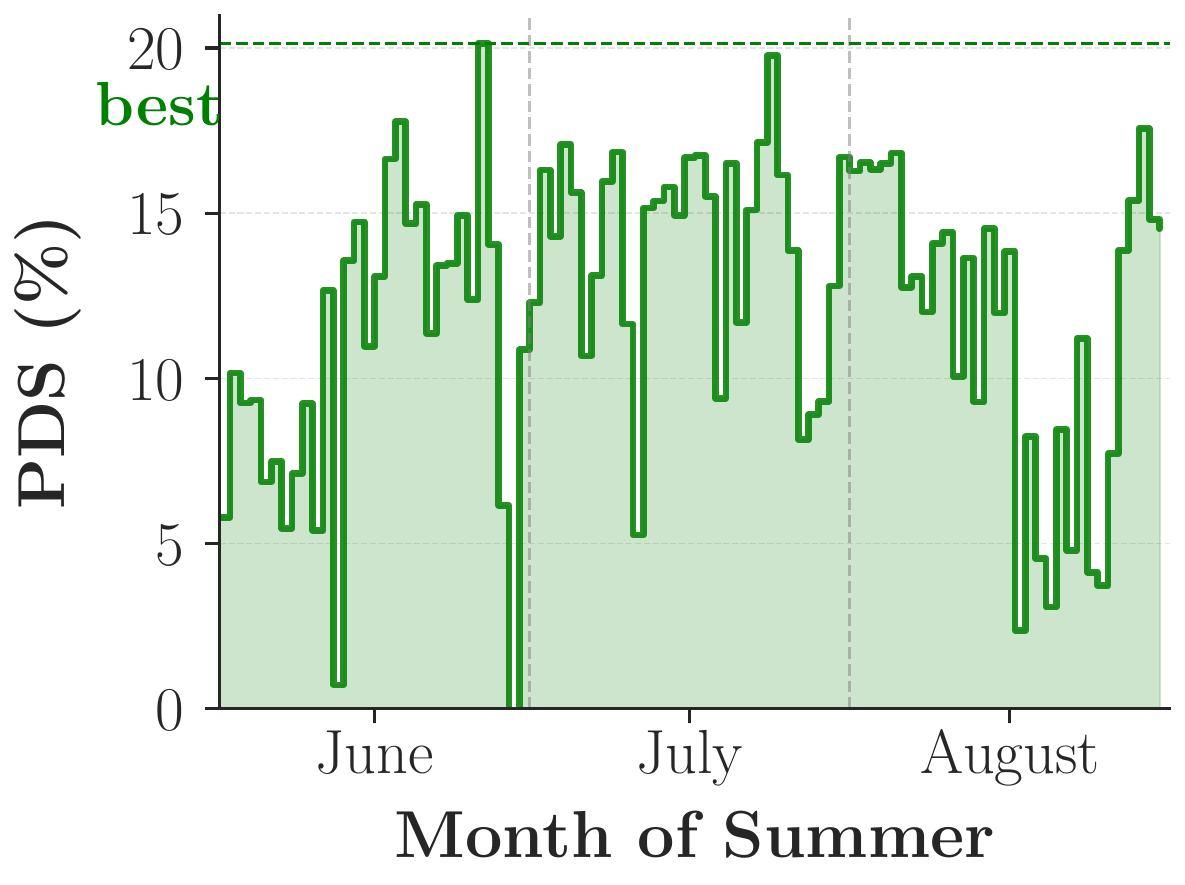}
        \caption{Atlanta}
    \end{subfigure}
    \hfill
    \begin{subfigure}{0.32\linewidth}
        \centering
        \includegraphics[width=\linewidth]{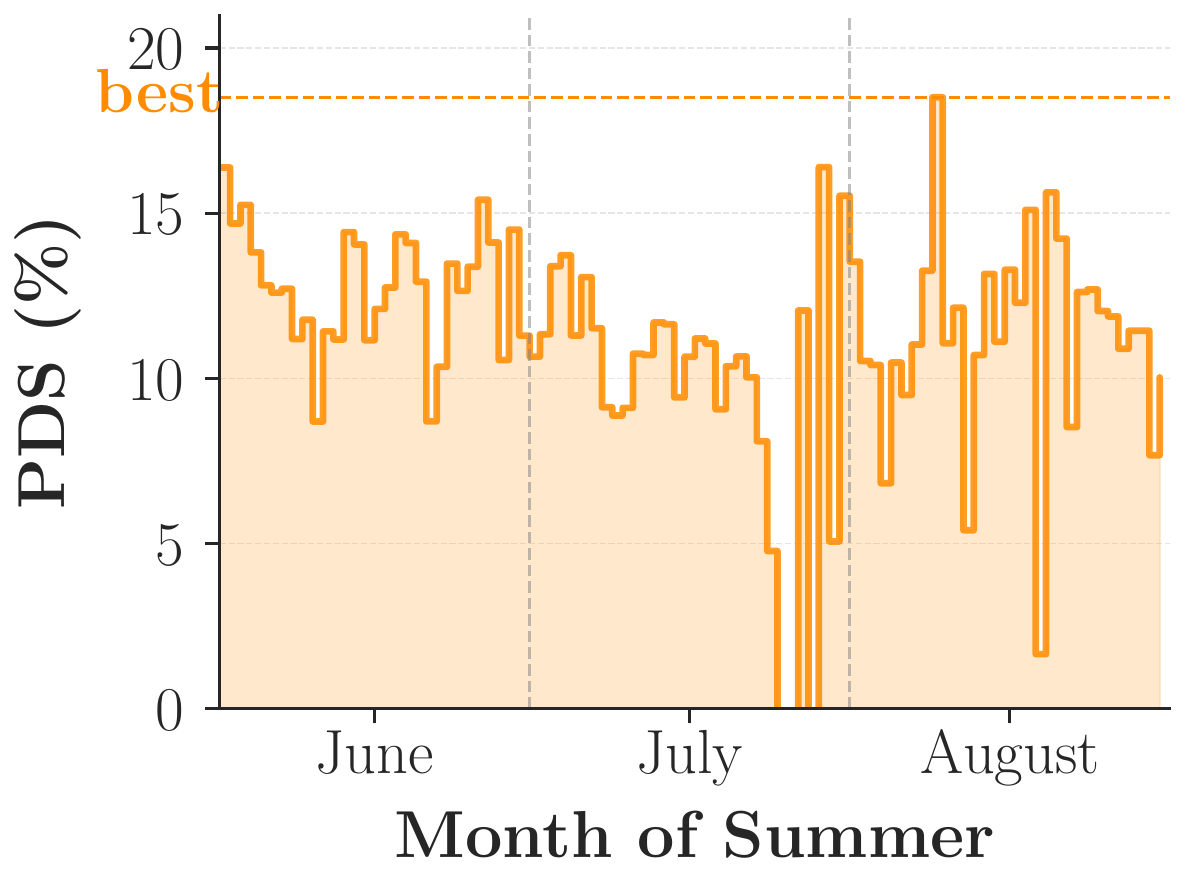}
        \caption{Phoenix}
    \end{subfigure}

    \caption{Daily peak demand shaving achieved by the cluster-based soft classifier pricing algorithm for San Francisco, Atlanta, and Phoenix in summer 2022.}
    \label{fig:PDS3cities}
\end{figure}

\vspace{3mm}
\autoref{fig:price3cities} presents the price signals generated by the cluster-based soft classifier algorithm, averaged over 30 days in July for each city. Several patterns emerge from this comparison. Most notably, both the intraday shape and the relative level of the average price profiles closely mirror the corresponding average hourly temperature patterns in each location. This alignment suggests that the temperature plays a particularly influential role in determining optimal pricing rule. In cities where temperatures rise sharply during certain hours, the algorithm assigns relatively higher prices during those periods, reflecting elevated cooling demand. Conversely, the average price profile remains comparatively flatter in milder climates. These patterns provide empirical evidence that the contextual learning mechanism successfully internalizes weather-driven variation when constructing price signals.
\begin{figure}[t]
    \centering
    \includegraphics[width=0.75\linewidth]{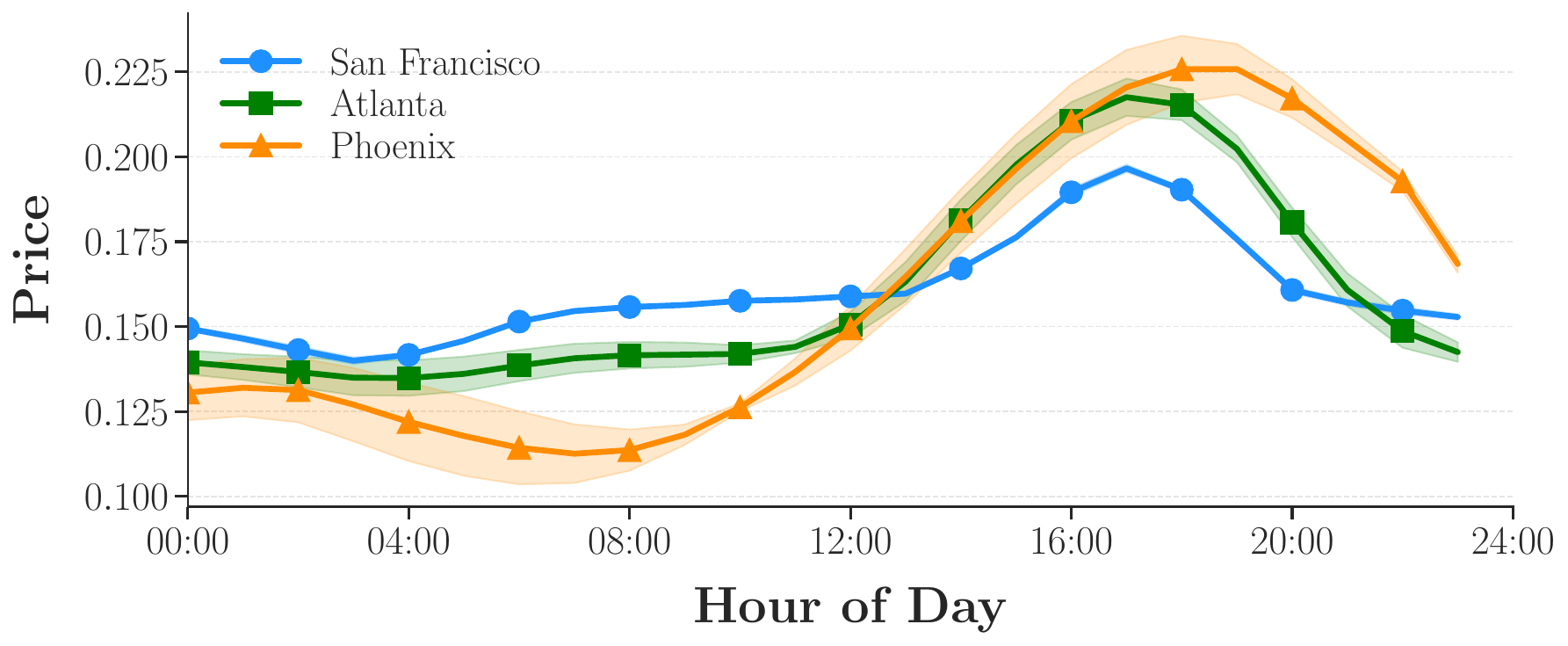}
    \caption{Average hourly price signals generated by the cluster-based soft classifier algorithm for San Francisco, Atlanta, and Phoenix in July 2022. Solid lines represent the mean hourly price profiles, and shaded regions represent $\pm 1$ standard deviation across days.}
    \label{fig:price3cities}
\end{figure}

\section{Theoretical properties} \label{sec:approx}

We establish two theoretical properties of the proposed architecture: its expressive power, through a universal approximation guarantee, and the convergence of its learning dynamics to stationary points of an underlying clustering problem.

\subsection{Universal approximation}
In this section, we show that any constrained mapping $p: \mathcal{X}\mapsto \mathcal{P}$ can be approximated arbitrarily well by the proposed architecture, provided sufficiently many clusters and hidden units are available.

We begin with invoking the theorem by \cite{asadi2020approximation} that shows the expressivity of $\text{softmax}(\cdot)$ as a lemma.

\begin{lemma} \label{lem:softmaxapprox}
    Let $I(X) = [I_1(X) \cdots I_{K}(X)]$ be a vector-valued indicator function, such that $I_i(X) = 0$ or $1$ for all $i$, $\sum_{i=1}^{K} I_i(X) = 1$ for all $X$.
    Then given any indicator function $I(X)$, assume each component function $I_i(X) \in L^1$, there exists some vector $h(X) = [h_1(X) \cdots h_{K}(X)]$ where $h_i(X) = \sum_{j=1}^n \alpha_{ij} \cdot g(\mathbf{w}_j^T X + b_j)$ for all $i$, such that their corresponding softmax outputs satisfy
    \begin{equation*}
        ||\textnormal{softmax}(h(X))_i-I_i(X)||_1 <\epsilon, \quad  \textnormal{for all } i, \;  \textnormal{for any small }\epsilon > 0, 
    \end{equation*}
     where $\textnormal{softmax}(h(X))_i = \frac{\exp\left(h_i(X)\right)}{\sum_{i=1}^{K} \exp\left(h_i(X)\right)}$, and $g$ can be taken from a variety of activation functions, such as $\text{sigmoid}(\cdot)$ or $\text{ReLU}(\cdot)$.
\end{lemma}  

\cref{lem:softmaxapprox} states that a sufficiently large neural network with $\text{softmax}(\cdot)$ on the output layer can well approximate any indicator target function in $L^1$. Now it suffices to show that any function $p: \mathcal{X}\mapsto \mathcal{P}$ can be approximated by $p(X) = {Z} \cdot I(X)$, where $I(X)$ is some vector-valued indicator function, $p(X) = [p_1(X) \cdots p_d(X)]$, ${Z} \in \{{Z}| {Z}_{\cdot j} \in \mathcal{P}, \forall j\}$. We invoke the following theorem of simple function approximation \cite[Theorem 13.5]{Billingsley1995-BILMAP} as a lemma.

\begin{lemma} \label{lemma:simpleapprox}
    Let simple function $m_k = \sum_{j=1}^k \alpha_j I_{A_j}$, where $I_{A_j}$ are indicator functions. If $m$ is real and measurable $\mathcal{F}$, then there exists a sequence $\{m_k\}$ of simple functions, each measurable $\mathcal{F}$, such that
    \begin{equation}
        \begin{aligned}
        0\leq m_k(\omega) \uparrow m(\omega),\quad \text{if }\, m(\omega) \geq0;\\
        0\geq m_k(\omega) \downarrow m(\omega),\quad \text{if }\, m(\omega) \leq0.
        \end{aligned}
    \end{equation}
\end{lemma}

The following theorem is then an extension of the lemma to vector-valued functions.

\begin{theorem}
\label{thm:univapprox}
Let $\mathcal{X}$ be a compact measurable space, and $p:\mathcal{X}\mapsto \mathcal{P}\subseteq \mathbb{R}^d$ be a measurable function, where $\mathcal{P}$ is compact. Then there exists a sequence of functions of the form
\begin{equation*}
    m_k(X) = {Z}^{(k)} \cdot I^{(k)}(X)
    = \sum_{j=1}^{N_k} {Z}^{(k)}_{\cdot j} I_{A^{(k)}_j}(X),
\end{equation*}
where $\{A^{(k)}_1,\dots,A^{(k)}_{N_k}\}$ is a measurable partition of $\mathcal{X}$, $I^{(k)}(X)$ is a vector-valued indicator function, and each column ${Z}^{(k)}_{\cdot j} \in \mathcal{P}$, such that
\begin{equation*}
    m_k(X) \to p(X), \quad \text{for all } X \in \mathcal{X}\text{ as } k \to \infty.
\end{equation*}
Furthermore, if $p$ is continuous, the convergence can be taken to be uniform.
\end{theorem}

A complete proof of this result is provided in Appendix \ref{appendix:univapprox}. To build intuition, the indicator $I(X)$ picks one column of ${Z}$ each time, where each column lies in $\mathcal{P}$. Therefore, $m_k$ is a piecewise constant function over a partition of $\mathcal{X}$, where each region $A^{(k)}_j$ is assigned a value ${Z}^{(k)}_{\cdot j} \in \mathcal{P}$. As $k \to \infty$, the partition $\{A^{(k)}_j\}$ becomes increasingly fine, allowing $m_k(X)$ to match $p(X)$ more closely on each region. Consequently, $m_k$ can approximate $p$ arbitrarily well---pointwise for measurable $p$, and uniformly if $p$ is continuous on compact $\mathcal{X}$. Together, Lemma \ref{lem:softmaxapprox} and \cref{thm:univapprox} establish the universal approximation property of the architecture in \eqref{eq:softclass}.

\subsection{Convergence analysis}\label{subsec:convergence}
We now turn to the learning dynamics of our proposed architecture. Specifically, we consider the following hard-constraint learning problem, where we are given $n$ samples $\{(x_i,y_i)\}_{i=1}^n$ with covariates $x_i\in\reals^p$ and responses constrained to a compact convex set, $y_i\in\cP\subset\reals^d$. We seek to explain response vectors with $K$ prototypes $Z=[z_1\ \cdots\ z_K]$, $z_r\in\cP$, and mixture weights $s_i\in\Delta_{K-1}$. Without considering context information, a natural target is to minimize the following \emph{context-free} objective function
\[
\ell_n(S,Z)=\ \frac{1}{2n}\big\|SZ^\sT-Y\big\|_{\mathrm{F}}^2+\frac{\tau}{n}\,\Tr\big(S^\sT\log S\big),
\]
where $Y$ stacks the responses $y_i^\sT$ and $S$ the weights $s_i^\sT$. Here, the penalty term with $\tau>0$ regularizes the weights of $S$. Specifically, it keeps them in the interior of the simplex and, in the overparametrized regime of $K>d$, makes the optimal weights unique for any fixed set of prototypes.

We next move to learning a mapping from contexts $x$ to response vector $y$, where we generate the mixture weights through the shared context-dependent network, and train the network parameters and the prototypes $Z$ by projected first-order updates. More precisely, we focus on the class of two-layer neural networks with activation function
$\sigma:\mathbb{R}\to\mathbb{R}$. For $N$ neurons, first-layer
parameters $U\in\mathbb{R}^{p\times N}$, output-layer parameters
$\th\in\mathbb{R}^{N\times(K-1)}$, we consider the following function
\begin{equation*}
f(x;U,\th,Z)
=
Z\,\mathsf{soft}\big(\th^\sT\sigma(U^\sT x)\big)\,.
\end{equation*}
Our updates for this network read as the following iterations:
\begin{align*}
(U_{t+1},\th_{t+1})
&=(U_t,\th_t)-\eta_t\nabla_{U,\th}\ell_n\big(\mathsf{soft}(\sigma(XU_t)\th_t),Z_t\big)\,, \nonumber\\
Z_{t+1}
&=\mathsf{proj}_{\cP^K}\Big(Z_t-\eta_t \nabla_{Z}\ell_n\big(\mathsf{soft}(\sigma(XU_t)\th_t),Z_t\big)\Big) \,.
\end{align*}

In Appendix~\ref{app:convergence},
we show that this restriction to learning a mapping from contexts to response vectors will still converge to a first-order stationary point of the following context-free clustering problem
\[
\min_{S,Z} \ell_n(S,Z)\,,\quad \text{ s.t. } S\in \Delta_{K-1}^n\,, Z\in \cP^{K}\,.
\]
 More precisely, we show that under mild regularity conditions on the first-layer features, the pairs $(S_t,Z_t)$ induced by the network along the trajectory, with weights $S_t=\mathsf{soft}\big(\sigma(XU_t)\th_t\big)$ plugged in from the current parameters, converge to a first-order stationary point of this context-free problem (Theorem~\ref{thm:stationary-points}). In addition, when its stationary points are isolated, the iterates converge to a single such point (Corollary~\ref{cor:isolated}). In other words, even though the weights are not optimized freely but are constrained to the image of our shared contextual parametrization, the updates still recover the stationarity guarantees of the underlying clustering problem.

\section{Discussion}\label{sec:discussion}
Motivated by the operational challenges faced by modern electricity grids, we study how demand response pricing can be deployed effectively in realistic, large-scale systems. To assess this question, we extend the existing demand response pricing mechanisms by incorporating rich contextual information into the computation of dynamic prices. Building on the mean-field characterization of Stackelberg equilibrium, we develop a context-aware learning framework that offers a principled way for demand response pricing in large-scale electricity systems. It improves over flat rate and context-agnostic prices particularly in environments with high renewable penetration or significant day-to-day weather fluctuations, highlighting its practical value.

On the other hand, from a methodological perspective, our approach illustrates how the dual characterization of Stackelberg equilibrium enables scalable learning. Rather than repeatedly solving a bi-level optimization problem for each observed state, we learn a constrained mapping from contextual signals to feasible prices. Once trained, this mapping can generate near-optimal prices with minimal computational overhead. This distinction is important for real-time or large-scale implementations, where repeated equilibrium computation may be impractical. Besides these benefits, such formulation also introduces algorithmic challenges of constrained output learning. Our proposed cluster-based soft classifier algorithm provides one way to balance computational efficiency and adaptability to constraints.

Several limitations remain. First, our framework relies on the accuracy and representativeness of contextual features; mis-specified or poorly predicted exogenous factors may degrade performance. In addition, while mean-field theory provides asymptotic guarantees, finite-population effects may cause behavior synchronization and introduce deviations from optimality.

Future work could focus on developing robust models that explicitly account for uncertainty in contextual features and produce interval or probabilistic pricing rules. For instance, integrating probabilistic forecasts or Bayesian neural networks for exogenous factors would quantify forecast uncertainty in the inputs. This allows the framework to be more robust against mis-specification and make risk-aware decisions. Another interesting direction is to extend the framework to incorporate strategic consumer behavior beyond passive price response. For example, consumers could be modeled as forward-looking agents who optimize intertemporal utility under comfort constraints, while anticipating pricing policies. Such extensions would enable the framework to capture endogenous demand shifting due to strategic behavior, leading to more sophisticated equilibrium.



\vspace{5mm}
\noindent\textbf{Funding.}
This work was supported by Trellis Climate (Prime Coalition)
and by the Office of Naval Research under grant number N00014-24-1-2091.

\bibliographystyle{apalike}
\bibliography{bibliography}       


\newpage

\appendix

\section{Approaches for constrained-output learning} \label{appendix:algreview}

This appendix proposes a cluster-based soft classifier approach that directly incorporates output feasibility while remaining computationally scalable. It also reviews several alternative approaches for constrained-output learning in neural networks, i.e. methods that enforce prediction feasibility of the form $\pi_\theta(X)\in \mathcal{P}$.

\vspace{3mm}
To illustrate this better, we consider the following structure of the neural network $\pi_\theta$:

\begin{assumption} \label{nnstruc}
We assume $\pi_\theta$ is from the function class $\Pi_\theta$ with the structure
    \begin{equation*}
    \begin{aligned}
    \Pi_\theta = \{&\pi_\theta: \mathcal{X} \mapsto \mathbb{R}^{d}, \, \pi_\theta \text{ is a fully connected neural network,}\\
    & \text{with 1 hidden layer and activation function $f$, } \pi_\theta(X) \in \mathcal{P}\}.
    \end{aligned}
\end{equation*}
Equivalently, $\pi_\theta$ can be parametrized by $\theta = (w_1,\, b_1,\, w_2,\, b_2) \in \Theta$, $\pi_\theta = w_2\,f(w_1X+b_1)+b_2$. For example, $f$ may be chosen as the sigmoid function, defined by $\text{sigmoid}(z) = 1/(1+e^{-z})$ applied entry-wise.
\end{assumption}

Prior work on constrained prediction in neural networks, i.e. $\pi_\theta(X) \in \mathcal{P}$, primarily falls into three categories: projection-based methods, barrier/interior-point methods, and proximal primal-dual methods. A detailed discussion of these approaches, along with a clear comparison of their strengths and limitations, is provided later in this section. We summarize these existing approaches together with our proposed approach below:
\begin{enumerate}[label=(\alph*)]
    \item \textit{Projection Layer Approach} \citep{Amos2017}: train neural network with no constraints, then add an additional layer to project the original output from neural network onto the constraint set.
    \item \textit{Interior Point Approach} \citep{kervadec2022constrained}: add log barrier term to the loss function, which actively imposes the constraints and has known efficient algorithms.
    \item \textit{Proximal Gradient Approach} \citep{nandwani2019primal}: apply proximal gradient descent to a weak dual of the primal problem, which allows for a larger range of non-differentiable constraints.
    \item \textit{Soft Classifier Approach}: train neural network as a mediating soft classifier that maps covariates to a set of weights on clusters, while simultaneously train the cluster-level outputs with constraints.
\end{enumerate}

The remainder of this section first introduces the structure of our proposed framework and the corresponding algorithmic updating procedure, then followed by the discussion on the other alternative approaches.

\subsection{Soft classifier architecture}\label{subsec:soft}
We present a cluster-based soft classifier algorithm that addresses major limitations of existing constrained-output learning approaches. The core idea is to train the neural network as a soft classifier, which outputs a probability distribution over learned clusters representing archetypal system conditions. Constraints are then imposed directly on the cluster-level outputs rather than on the network outputs or internal parameters.

Such design avoids the need to convert the feasible set from the output space $\mathcal{P}$ to the network parameter space $\Theta$, and requires only projected gradient descent to ensure feasibility. By structuring the learning in this way, the algorithm maintains scalability, improves convergence, and reduces overfitting risk compared with naïve projection approaches. Formally, we define the neural network architecture as
\begin{equation}\label{eq:softclass}
     p_t = {Z} \cdot \psi(X_t),\quad \psi(X_t) = \text{softmax}\left\{\left(w_2f\left(w_1X_t+b_1\right)\right)+b_2\right\},
\end{equation}
where the neural network output is $\psi(X)\in \mathbb{R}^k$, with $k$ denoting the total number of clusters; the cluster-level output is ${Z} \in \mathbb{R}^{d\times k}$; and the final output $p \in \mathcal{P}\subset\mathbb{R}^d$. \autoref{fig:soft_class} illustrates the structure of the cluster-based soft classifier algorithm.

\begin{figure}[t]
    \centering
    \includegraphics[width=0.65\linewidth]{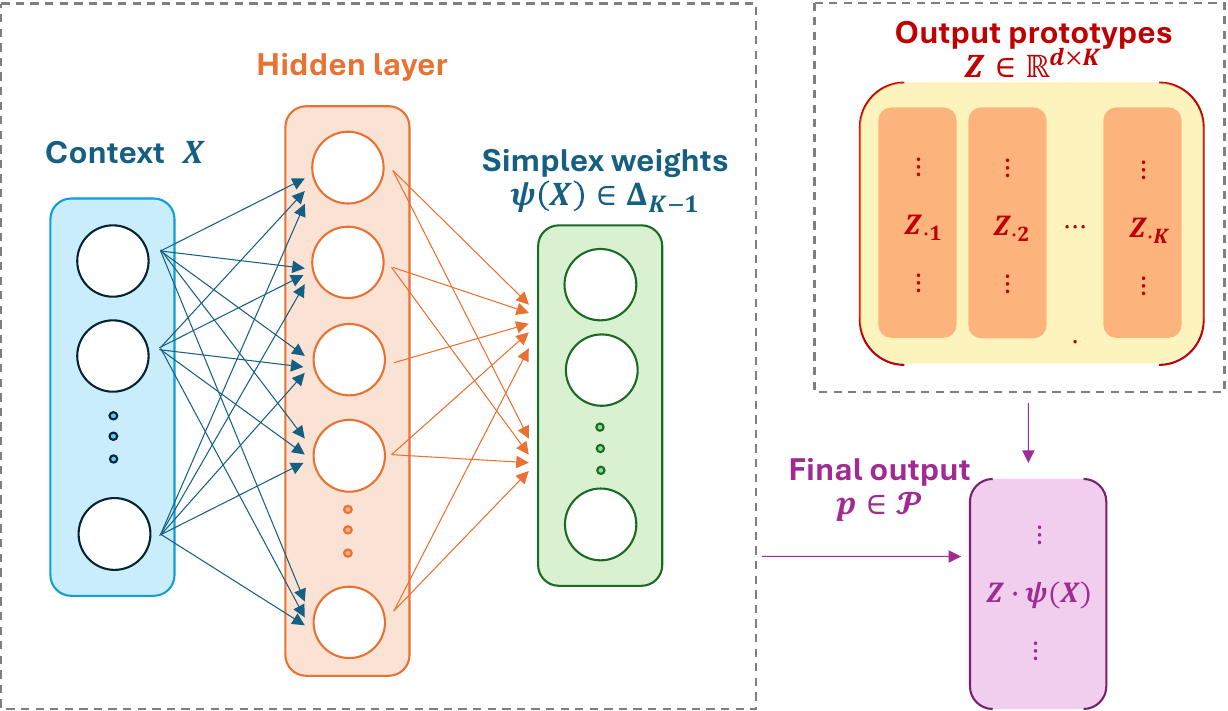}
    \caption{Constrained output learning architecture with cluster-based soft classifier. Network maps input features to a probability assignment over clusters, with each cluster representing a archetypal system condition. Constraints are applied directly to the cluster-level outputs, ensuring that the final price signal---computed as a convex combination of the cluster-level outputs---remains feasible.}
    \label{fig:soft_class}
\end{figure}

\vspace{3mm}
The following lemma establishes the equivalency of constraints on the final output and cluster-level outputs.

\begin{lemma}
Under the learning structure in \eqref{eq:softclass}, the following statements are equivalent:
    \begin{equation*}
    p_t = {Z} \cdot \psi(X_t) \in \mathcal{P}\Leftrightarrow {Z} \in \{{Z}| {Z} \cdot v \in \mathcal{P},\, \forall v \in \Delta_{K-1}\}\Leftrightarrow {Z} \in \{{Z}| {Z}_{\cdot j} \in \mathcal{P}, \forall 1\leq j \leq K\},
    \end{equation*}
where $\Delta_{K-1}\subset \mathbb{R}^{K}$ is the $(K-1)$-dimensional probability simplex. In particular, the final condition requires that each column of ${Z}$ satisfies the constraints.
\end{lemma}

\subsection{Algorithmic update}
We now describe the optimization procedure for the proposed learning framework. The model $\psi(x)$ and the cluster-level parameters ${Z}$ are updated sequentially using different rules. Specifically, the parameters $(w_1, w_2, b_1, b_2)$ of $\psi(x)$ are updated via unconstrained gradient descent, while the constrained cluster-level outputs ${Z}$ are updated via projected gradient descent (PGD). The full procedure is presented in \cref{alg1}, where the gradient of the loss function $l(w_t,b_t,{Z}_t)$ follows from Theorem~\ref{thm:grad}.

We emphasize that the use of projection in conjunction with batch gradient descent, as illustrated in the pseudo-code, represents one possible implementation of the projected gradient descent. More broadly, a wide class of optimization methods can be adapted to this framework, including stochastic gradient descent, mirror descent, and momentum-based methods such as Adam and Nesterov accelerated gradient descent.

\begin{algorithm}[t]
\caption{Cluster-based Soft Classifier Algorithm}\label{alg1}
{\renewcommand{\baselinestretch}{1.4}\selectfont
\begin{algorithmic}
\State
\State $w,\,b,\,{Z} \gets w_0,\,b_0,\,{Z}_0$ \Comment{Initialize training variables}
\For{$t$ in $1:T$}
\State $w_t \gets w_{t-1} + \eta_t \nabla_w l(w,b_{t-1},{Z}_{t-1})$
\State $b_t \gets b_{t-1} + \eta_t \nabla_b l(w_{t-1},b,{Z}_{t-1})$\Comment{Train model $\psi(x)$ without constraints}

\State
\State ${Z}_t' \gets {Z}_{t-1} + \eta_t \nabla_{Z} l(w_{t-1},b_{t-1},{Z})$ 
\State ${Z}_t \gets \text{Proj}_\mathcal{P}( {Z}_t')$
\Comment{Train clusters ${Z}$ using projected gradient descent}
\EndFor
\State
\end{algorithmic}}
\end{algorithm}

\subsection{Projection layer}
To enforce constraints on the output, one natural approach is to incorporate a projection layer directly within the network and back-propagate through its gradient. This method offers several advantages: it guarantees that all outputs satisfy the constraints during training, can be readily integrated into standard neural network architectures, and allows efficient gradient-based updates.

As illustrated in \autoref{fig:proj_layer}, the learning architecture is composed of four layers that collectively implement the constrained output learning. The key distinction between the projection layer approach and projected gradient descent lies in where the projection is applied: the former projects the network outputs directly, while the latter applies projection to the training variables iteratively as part of the optimization routine. This difference has important implications for practicality under output constraints. In particular, the projection layer method is well-suited for enforcing constraints directly on predicted outputs, enabling straightforward feasibility and gradient-based updates. In contrast, projected gradient descent operates on the training variables themselves, making it unclear how to impose the desired constraints.

However, the projection layer approach also has notable limitations. First, the projection operator is generally non-differentiable (e.g. projecting onto an $l_1$-ball), which makes gradient computation challenging. Second, training an unconstrained neural network followed by projection can increase the risk of overfitting, as the network may learn complex internal representations that satisfy the constraints only after projection, rather than naturally producing feasible outputs.

\begin{figure}[t]
    \centering
    \includegraphics[width=0.6\linewidth]{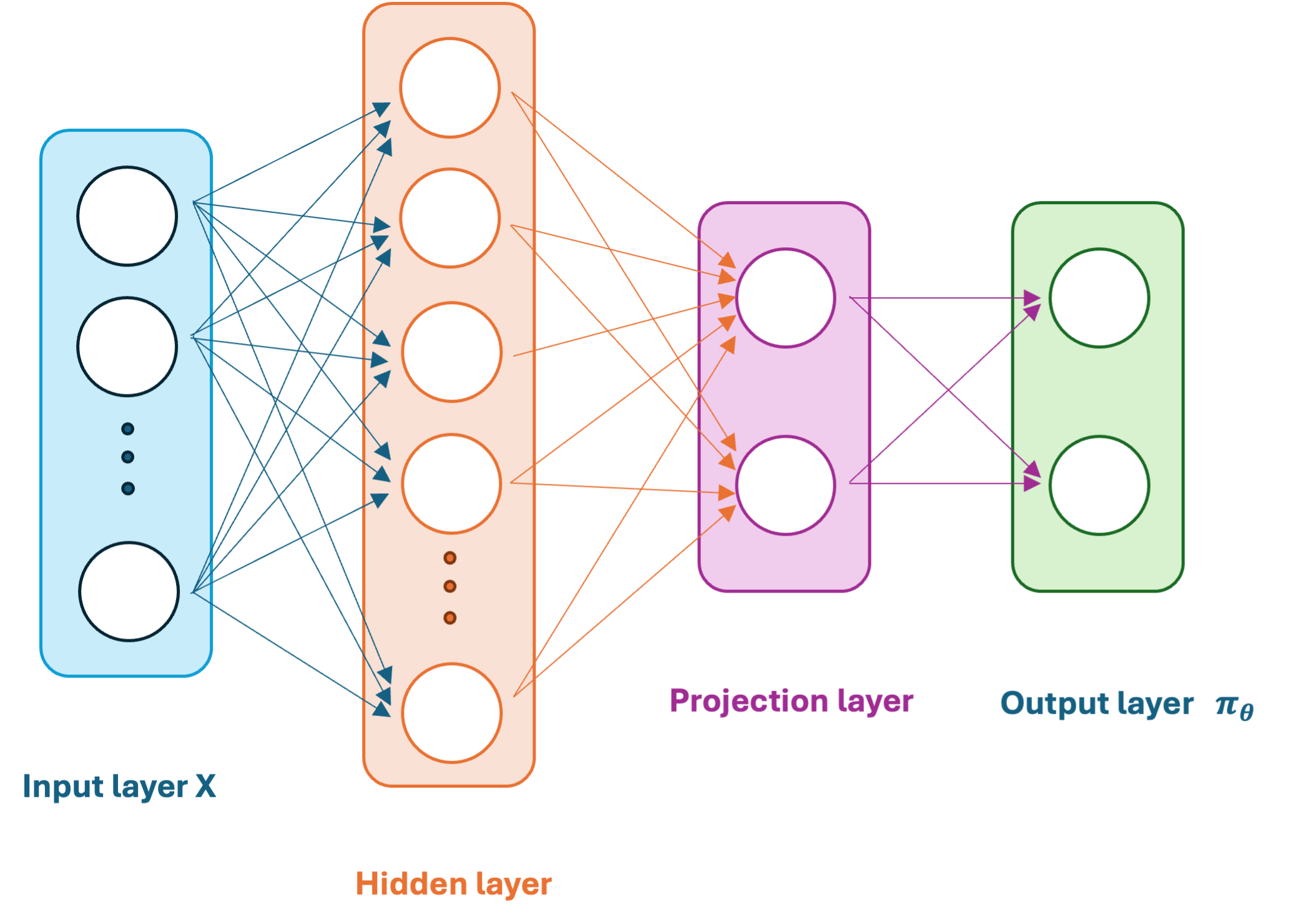}
    \caption{Constrained output learning architecture with a projection layer. Unlike projected gradient descent, which adjusts training variables iteratively, the projection layer approach enforces constraints directly on the network outputs by projecting them onto the feasible set.}
    \label{fig:proj_layer}
\end{figure}

Some remedies to overfitting of the unconstrained neural network are as follows. One option is to add a dropout layer as in \cite{srivastava2014dropout}, which randomly sets some nodes to output zero during training. Another option is to use mirror gradient descent (MGD), which could restrict the search space of parameter update and accelerate the convergence. For example, if we set mirror map to be the negative entropy $h(x)=\sum_{i=1}^nx_i \ln x_i$, then the update preserves the sign at each update step, as $(x_{t+1})_i = (x_t)_ie^{
-\eta\nabla_i f(x_t)}$.

\subsection{Interior point}
Another common approach for enforcing output constraints while mitigating over-parameterization is the interior point method. Formally, an additional penalty term $\lambda log(1-\rho^*(\pi_\theta(X_t))$, also known as \textit{log barrier function}, is added to the loss:

\begin{equation}
    \begin{aligned}
    &\max_{p \in \mathcal{P}} \mathbb{E} [\rho\left(q_t(p_t)-Q^0(X_t)\right)] +\lambda log(1-\rho^*(p))\\
     =&\max_{\theta \in \Theta} \mathbb{E} [\rho\left(q_t(\pi_\theta(X_t))-Q^0(X_t)\right)]+\lambda log(1-\rho^*(\pi_\theta(X_t))),
    \end{aligned}
\end{equation}
where $\lambda \in\mathbb{R}_{>0}$ is the penalty level, and the log barrier approaches $+\infty$ when $\rho^*(\pi_\theta(X_t)) \to 1$. In practice, one drawback of the interior point method is its numerical instability and ill-conditioned KKT systems near boundary. Tuning the hyperparameter $\lambda$ is also challenging, as it is closely related to the central path of the underlying optimization problem; see \cite{boyd2004convex} for a detailed discussion.

Another drawback is that the method still fails when the barrier function is not differentiable. Nevertheless, if we consider a neural network with an additional non-negative activation applied to the output layer (e.g. sigmoid), then the non-differentiable points associated with projection onto the $l_1$-ball lie outside the range of the output. Hence, the algorithm may still be applied. More generally, when the constraint is non-differentiable, subgradient methods can be utilized, even though they are known to have slower convergence rates.

\subsection{Proximal gradient}
Proximal gradient descent shares a similar formulation with interior point methods, but it accommodates a broader class of constraint functions, including non-differentiable ones. The main idea is to utilize weak Lagrange duality, which gives the best lower bound on the original target function. This holds even if the original problem is not convex. Formally, the Lagrange dual is defined as follows, with $\lambda\in\mathbb{R}_{\geq0}$:

\begin{equation} \label{eq:Lagdual}
    \min_{\theta} -\mathbb{E} [\rho\left(q_t(\pi_\theta(X_t))-Q^0(X_t)\right)]+\lambda (\rho^*(\pi_\theta(X_t)) -1).
\end{equation}

Since the dual problem in \eqref{eq:Lagdual} is still hard to solve, we utilize a secondary dual. The idea is similar to deriving the solution path for generalized Lasso problem in \cite{tibshirani2011solution}, where the $l_1$-penalty is composed with a linear transformation of $\beta$. We start with rewriting the problem in \eqref{eq:Lagdual} as

\begin{equation*}
\min_{z \in \mathbb{R}^d,\, \theta} -\mathbb{E}[\rho\left(q_t (\pi_\theta(X_t))- Q^0(X_t)\right)]+\lambda (\rho^*(z) -1) 
\quad s.t.\; \pi_\theta(X_t) =z,
\end{equation*}
which gives the secondary Lagrange dual
\begin{equation*}
\min_{z \in \mathbb{R}^d,\, \theta} -\mathbb{E} [\rho\left(q_t(\pi_\theta(X_t))-Q^0(X_t)\right)]+\lambda (\rho^*(z) -1) + u(\pi_\theta(X_t)-z).
\end{equation*}

The advantage of separating training variables $z$ and $\theta$ is that it simplifies the gradient for each variable. Now the optimal solution of $\theta$ can be defined as a function of $u$:
\begin{equation} \label{dualtarget}
    g^*(u) = \min_\theta \left\{-\mathbb{E} [\rho\left(q_t (\pi_\theta(X_t))-Q^0(X_t)\right)]+  u\pi_\theta\right\}.  
\end{equation}

On the other hand, the terms that do not involve $\theta$ can be solved by
\begin{equation} \label{dualconstraint}
\min_{z \in \mathbb{R}^d} \lambda (\rho^*(z) -1) -uz = \begin{cases}
    -\lambda, &\text{if $\rho(u) \leq \lambda$}; \\
    -\infty, &\text{if $\rho(u) > \lambda$}.
\end{cases}
\end{equation}

Finally, combining \eqref{dualtarget} and \eqref{dualconstraint} yields the dual formulation of \eqref{eq:Lagdual}:
\begin{equation} \label{dual}
\max_{u \in \mathbb{R}^d} g^*(u) \quad s.t.\; \rho(u) \leq \lambda.
\end{equation}

We again utilize the mean-field solution in \eqref{eq:mean-field-price-characterization} to compute the gradient with respect to $\theta$, and employ proximal gradient descent to update the dual variable $u$, projecting it onto the box $[-\lambda, \lambda]^d$. By nature of the problem, an alternative gradient-based method is an ideal candidate, such as the gradient descent ascent (GDA). Some theoretical analyses of GDA can be found in the literature: \cite{NIPS2017_8a1d6947} proposed the popular two-time scale GDA with convergence guarantees beyond the convex-concave setting; \cite{pmlr-v119-lin20a} provided nonasymptotic complexity result of two-time scale GDA and stochastic GDA, etc.

\vspace{3mm}
Lastly, \cref{tab:summary} summarizes the advantages and limitations of the three existing approaches discussed in this appendix, and compares them with the proposed framework in the main text.
\begin{table}[t]
\begin{center}
\caption{Advantages and Limitations to Different Methods}
\label{tab:summary}
\begin{tabularx}{\linewidth}{|l|X|X|X|} 
 \hline
 \textbf{Method} & \textbf{Convergence}& \textbf{Adaptiveness} & \textbf{Efficiency}\\
 \hline
  Projection layer & \text{\sffamily x} Leads to over-parametrization & \text{\sffamily x} Requires differentiable projection operator & \checkmark \hspace{0.3mm} Utilizes various first-order methods \\
 \hline
 Interior point & \checkmark \hspace{0.3mm} Imposes constraints actively & \text{\sffamily x} Requires differentiable norm as constraints  & \text{\sffamily x} Causes numerical instability near boundary\\ 
 \hline
 Proximal gradient& \text{\sffamily x} Provides only lower bound via weak dual & \checkmark \hspace{0.3mm} Allows for constraints to be non-differentiable & \checkmark \hspace{0.3mm}  Combines gradient descent ascent (GDA) with proximal operators\\ 
 \hline
 Soft classifier & \checkmark \hspace{0.3mm} Imposes constraints via projected gradient descent; Admits universal approximation& \checkmark \hspace{0.3mm} Allows for constraints to be non-differentiable & \checkmark \hspace{0.3mm} Utilizes various first-order methods \\ 
 \hline
\end{tabularx}
\end{center}
\end{table}

\section{Extended case study analysis} \label{appendix:case}
We observe from \autoref{fig:PDS3cities} that our proposed learning framework achieves positive peak demand reduction consistently across cities. From a practical perspective, utilities implementing demand response programs are particularly concerned with whether potential increases in peak demand under dynamic pricing occur on high-demand days, as such outcomes could lead to excessive grid risk. To further quantify this effect, we present additional results on the absolute magnitude of net peak demand under flat-rate and context-aware demand response pricing. These results provide a more comprehensive evaluation of the proposed algorithm’s impact on grid risk.

\autoref{fig:scatter3cities} compares net peak demand under the two pricing schemes. We first observe substantial differences in peak demand across the three cities, which align with the temperature and solar irradiance patterns shown in \autoref{fig:weather3cities}, further confirming that demand is strongly influenced by weather conditions.

Additionally, reductions are consistent across days and remain pronounced during high-demand periods, highlighting the value of the demand response program from a grid reliability perspective. Although observations are relatively sparse at high demand levels, most points lie below the 45-degree line, suggesting that the method shows no evidence of exacerbating peak demand during critical periods. This behavior is particularly important for system operators, as it ensures that the demand response program does not introduce additional risk when the grid is most vulnerable.

\begin{figure}[t]
    \centering

    \begin{subfigure}{0.32\linewidth}
        \centering
        \includegraphics[width=\linewidth]{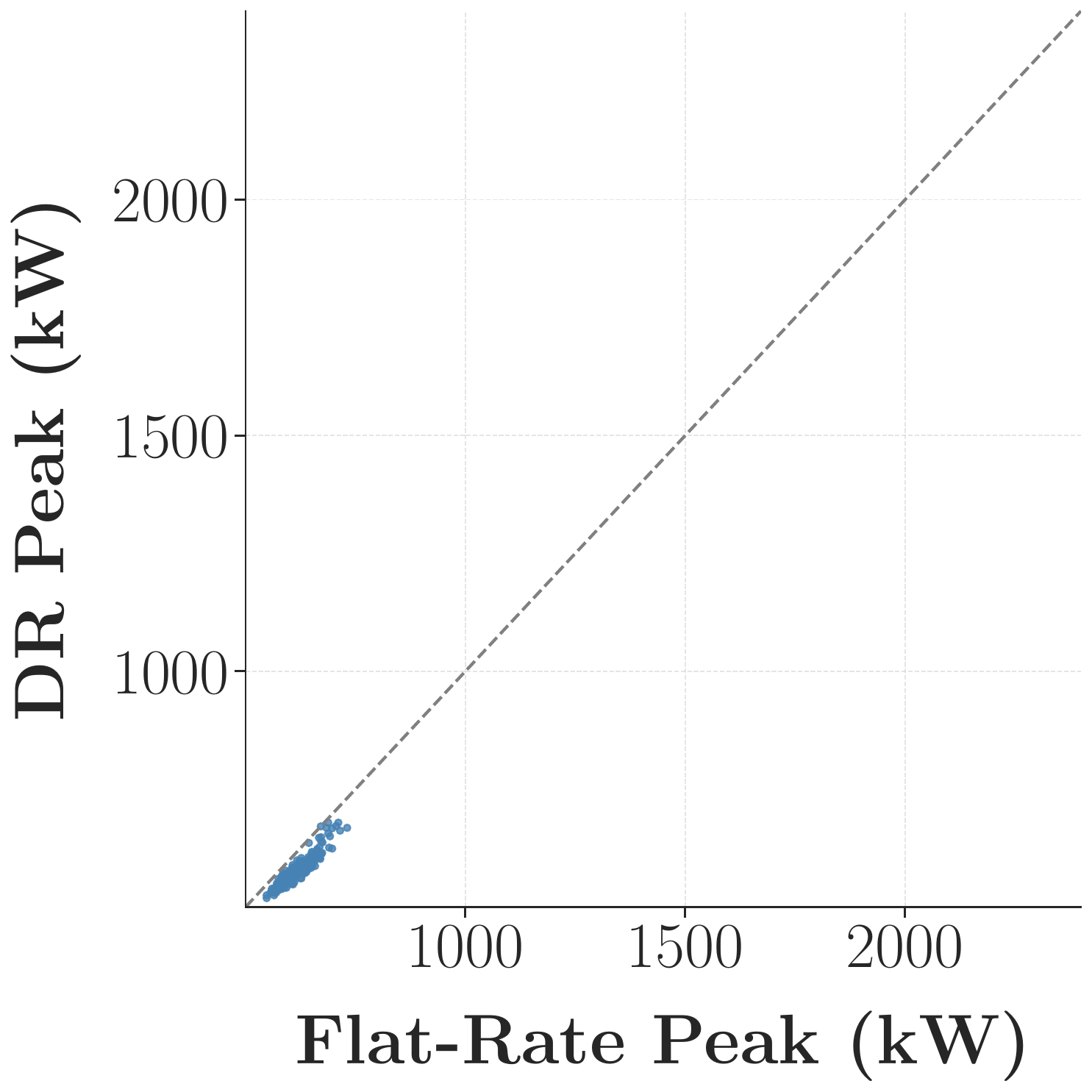}
        \caption{San Francisco}
    \end{subfigure}
    \hfill
    \begin{subfigure}{0.32\linewidth}
        \centering
        \includegraphics[width=\linewidth]{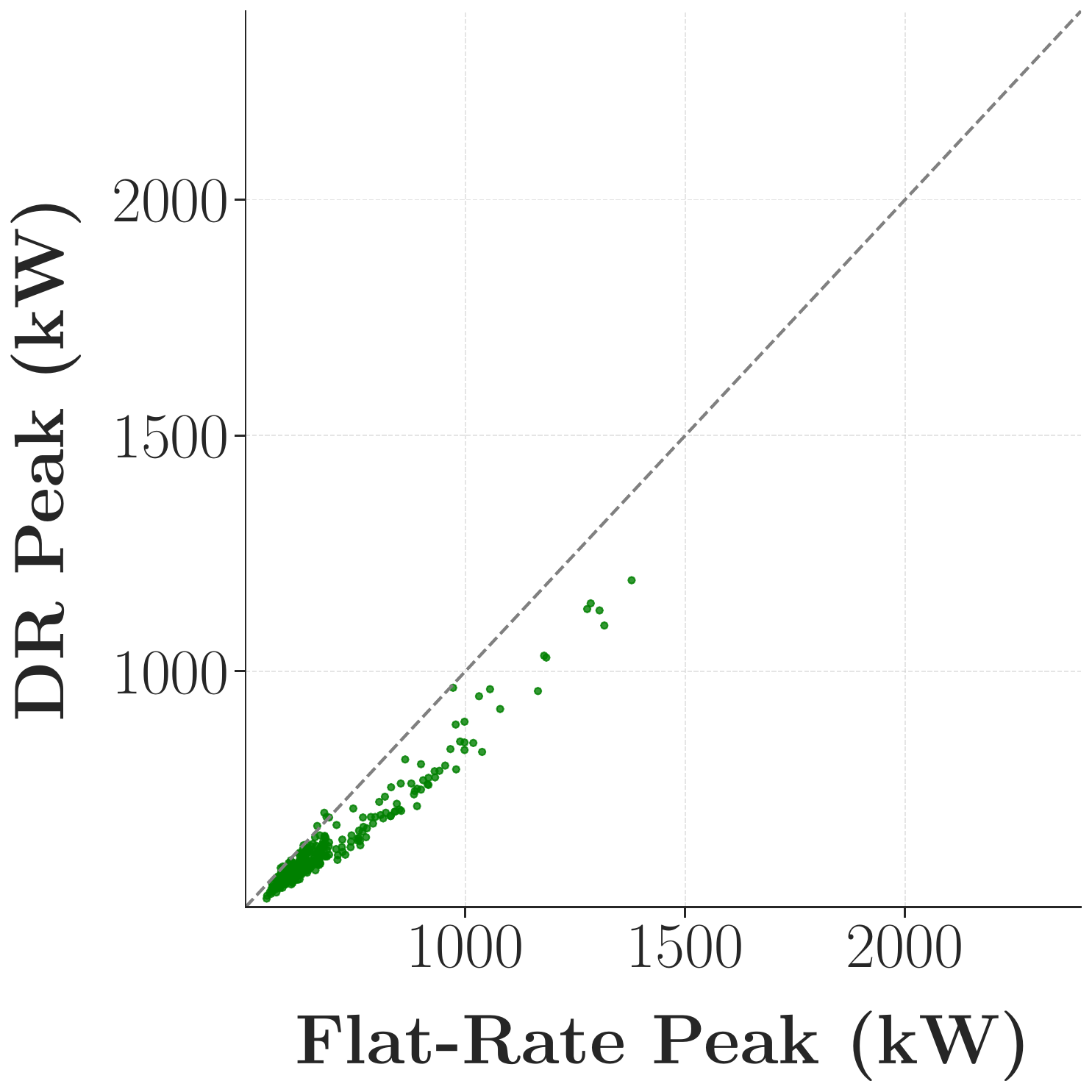}
        \caption{Atlanta}
    \end{subfigure}
    \hfill
    \begin{subfigure}{0.32\linewidth}
        \centering
        \includegraphics[width=\linewidth]{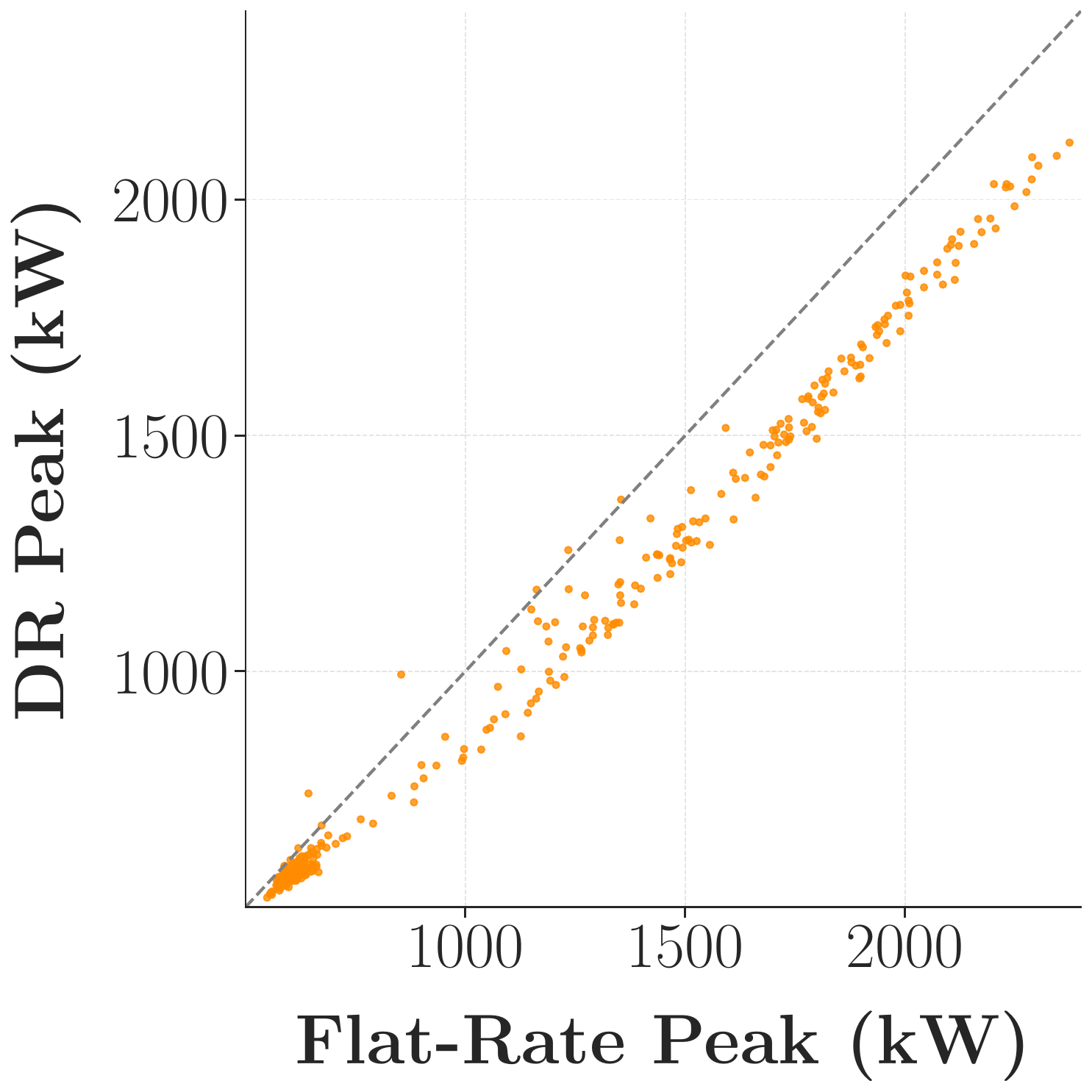}
        \caption{Phoenix}
    \end{subfigure}

    \caption{Comparison of net peak demand under flat-rate and context-aware demand response pricing across cities in 2022. The 45-degree line indicates identical net peak demand under the two pricing schemes. Points below the line correspond to days on which the proposed learning framework achieves lower net peak demand than the flat-rate scheme.}
    \label{fig:scatter3cities}
\end{figure}

\section{Proof of Theorem~\ref{thm:grad}}
\label{sec:grad}
In order to prove Theorem \ref{thm:grad} we first state a result borrowed from \cite{mehrabi2024optimalmechanismsdemandresponse} for the optimal price characterization based on average consumer cost functions, and a subgradient of this objective function based on consumers consumption profiles.  

\begin{theorem}[Theorem~4 of \citet{mehrabi2024optimalmechanismsdemandresponse}]
Under the indifference-set setup, for $n$ consumers let $\{q_i(p)\}_{i\ge 1}$ denote the $\hems$ responses given in \eqref{eq:HEMS}, where the cost of household $i\in[n]$ is $\smash{h_i(p):=p^\sT q_i(p)}$. In the mean-field limit, define the average consumer cost as $H(p):=\lim_{n\to\infty}\frac{1}{n}\sum_{i=1}^n h_i(p)$, and the per-consumer renewable supply as   $q_0:=\lim_{n\to\infty}\frac{Q_0}{n}$.
Then, for a grid risk function $\rho(\cdot)$ and its associated dual set $\mathcal{P}$,\footnote{For a convex, positively homogeneous risk function $\rho(\cdot)$, the associated dual set is
$\mathcal{P}=\{z\in\mathbb{R}^d: z\cdot q\le \rho(q), \ \forall q\in\mathbb{R}^d\}$, so that $\rho$ is the support function of $\mathcal{P}$ \cite[Claim A.2.24]{artstein2015asymptotic}. When $\rho$ is a norm, with dual norm $\rho^*(\cdot)$, this set admits the dual-norm representation
$\mathcal{P}=\{z\in\mathbb{R}^d:\rho^*(z)\le 1\}$.}
the optimal price is characterized by
\begin{equation}
\label{eq:basic}
    p^* \in \arg\max_{p \in \cP} \; H(p)-q_0\cdot p .
\end{equation}
In addition, the consumer cost function $H(\cdot)$ is concave in price $p$, and the mean-field average consumption response is a supergradient of it, formally,
\[
    \lim_{n\to\infty}\frac{1}{n}\sum_{i=1}^n q_i(p)\in \partial H(p).
\]
\end{theorem}


We now complete the proof of Theorem~\ref{thm:grad} by applying the chain rule to the price characterization above. For the finite-sample objective, define
\[
\widehat H_n(p):=\frac{1}{n}\sum_{i=1}^n h_i(p).
\]
By the supergradient characterization stated above, for any price $p$, the average selected consumption profile
$\frac{1}{n}\sum_{i=1}^n q_i(p)$ is a supergradient of $\widehat H_n(\cdot)$ at $p$. Therefore, for the price-dependent objective $\widehat H_n(p)-q_0\cdot p$,
a corresponding supergradient with respect to $p$ is
\[
\frac{1}{n}\sum_{i=1}^n q_i(p)-q_0.
\]

Next, by the definition of $\mathsf J_\theta(x)$ in \eqref{eq:J}, the transposed Jacobian of the induced price map $Z\pi_\theta(x)$ with respect to $\theta$ is $\mathsf J_\theta(x)Z^\sT$. Applying the chain rule gives
\begin{equation}
\begin{aligned}
\nabla_{\theta}
\left[
\frac{1}{n}\sum_{i=1}^n h_i(Z\pi_\theta(x))
- q_{0}\cdot Z\pi_\theta(x)
\right]
&=
\mathsf{J}_{\theta}(x)Z^\sT
\left(
\frac{1}{n}\sum_{i=1}^n q_i(Z\pi_\theta(x))
- q_{0}
\right),
\end{aligned}
\end{equation}
which proves \eqref{eq:chain-rule}. When the selected household responses are unique, this expression is the usual gradient. Otherwise, it should be interpreted as a valid supergradient associated with the selected optimal responses.

\section{Proof of Theorem~\ref{thm:univapprox}} \label{appendix:univapprox}
We prove that any measurable function $p:\mathcal{X}\to\mathcal{P}\subseteq\mathbb{R}^d$ can be approximated by functions of the form
\begin{equation*}
    m_k(X) = \sum_{j=1}^{N_k} {Z}^{(k)}_{\cdot j} I_{A^{(k)}_j}(X),
\end{equation*}
where $\{A^{(k)}_j\}_{j=1}^{N_k}$ is a measurable partition of $\mathcal{X}$, each ${Z}^{(k)}_{\cdot j}\in\mathcal{P}$, and $I_{A}(X)$ denotes the indicator function of a set $A$. We construct such a sequence $\{m_k\}$ so that $m_k(X)\to p(X)$ for all $X \in \mathcal{X}$ as $k \to \infty$.

\begin{proof}
We first construct a sequence of vector-valued simple functions approximating $p$. Since $p:\mathcal{X} \to \mathbb{R}^d$ is measurable, each coordinate function $p_\ell:\mathcal{X} \to \mathbb{R}$, $\ell=1,\dots,d$, is measurable. By Lemma~\ref{lemma:simpleapprox}, for each $\ell$ there exists a sequence of scalar simple functions $\{s_{\ell,k}\}$ such that
\begin{equation} \label{eq:simpleapprox}
    s_{\ell,k}(X) \to p_\ell(X), \quad \text{for all } X \in \mathcal{X}\text{ as } k \to \infty.
\end{equation}

Define the vector-valued function $s_k(X) := \bigl(s_{1,k}(X), \dots, s_{d,k}(X)\bigr)^\top$, then $s_k(X) \to p(X)$ pointwise. Since each $s_k$ is a simple function, it takes finitely many values. Hence there exists a measurable partition $\{A^{(k)}_1,\dots,A^{(k)}_{N_k}\}$ of $\mathcal{X}$ and vectors $v^{(k)}_1,\dots,v^{(k)}_{N_k} \in \mathbb{R}^d$ such that
\begin{equation*}
    s_k(X) = \sum_{j=1}^{N_k} v^{(k)}_j I_{A^{(k)}_j}(X).
\end{equation*}

Next, we ensure that the approximating values lie in $\mathcal{P}$. Since $\mathcal{P}$ is compact, for each $k$ and each $j$, we can select ${Z}^{(k)}_{\cdot j} \in \mathcal{P}$ such that
\begin{equation}\label{eq:compact}
    \|{Z}^{(k)}_{\cdot j} - v^{(k)}_j\| \leq \varepsilon_k, \quad \varepsilon_k \downarrow 0.
\end{equation}

Now define $m_k(X) := \sum_{j=1}^{N_k} {Z}^{(k)}_{\cdot j} I_{A^{(k)}_j}(X) = {Z}^{(k)} \cdot I^{(k)}(X)$, then $m_k(X) \in \mathcal{P}$ for all $X$. Combining \eqref{eq:simpleapprox} and \eqref{eq:compact}, for all $X \in \mathcal{X}$,
\begin{equation*}
    \|m_k(X) - p(X)\|
    \leq \|m_k(X) - s_k(X)\| + \|s_k(X) - p(X)\|
    \leq \varepsilon_k + \|s_k(X) - p(X)\|,
\end{equation*}
where $\varepsilon_k \to 0$ and $s_k(X) \to p(X)$. Therefore, we conclude that
\begin{equation*}
    m_k(X) \to p(X), \quad \text{for all } X \in \mathcal{X}.
\end{equation*}

Further, if $p$ is continuous and $\mathcal{X}$ is compact, then $p$ is uniformly continuous. In this case, we can construct the partition $\{A^{(k)}_j\}$ to be sufficiently fine and choose ${Z}^{(k)}_{\cdot j} = p(x^{(k)}_j)$ for some $x^{(k)}_j \in A^{(k)}_j$, which yields
\begin{equation*}
    \sup_{X \in \mathcal{X}} \|m_k(X) - p(X)\| \to 0.
\end{equation*}
This completes the proof.
\end{proof}

\section{Convergence analysis of prototype learning
}\label{app:convergence}

We consider the problem of hard-constraint learning from $n$ samples $\{(x_i,y_i)\}_{i=1}^n$, where we suppose that the covariates $x_i$ are bounded and drawn i.i.d.\ from some distribution in $\mathbb{R}^p$, and the response vectors satisfy $y_i\in\mathcal{P}$
for some compact convex set $\mathcal{P}\subset\mathbb{R}^d$. We consider the response matrix $Y\in\reals^{n\times d}$ whose $i$-th row is $y_i^\sT$, and let $S\in\reals^{n\times K}$ denote the mixture-weight matrix whose $i$-th row is $s_i^\sT\in\Delta_{K-1}$, and for $K$ prototypes we let $Z\in\reals^{d\times K}$ denote the prototype matrix whose $r$-th column is $z_r\in\cP$. Throughout this section, $\Delta_{K-1}=\{s\in\reals^{K}_{\ge 0}:s^\sT \mathbf{1}_{K}=1\}$ denotes the probability simplex on $K$ outcomes, and $\mathsf{soft}:\reals^{K-1}\to\Delta_{K-1}$ denotes the normalized softmax, so that for $u\in\reals^{K-1}$ the $r$-th entry of $\mathsf{soft}(u)$ is $e^{u_r}/(1+\sum e^{u_q})$ for $r\in[K-1]$, and its $K$-th entry is $1/(1+\sum e^{u_q})$. For simplicity of notation, we consider the set of feasible values $\cS=\Delta^{n}_{K-1}$ and $\cZ=\cP^{K}$. We start by considering the following context-free clustering problem, in which we seek the $K$ prototypes that best explain the response vectors.

\begin{equation}
\label{eq:context-free}
\min_{{Z\in\cZ\,,S\in\cS}}
\Big\{ \ell_n(S,Z):=
\frac{1}{2n}\big\|SZ^\sT-Y\big\|_{\mathrm{F}}^2
+\frac{\tau}{n}\,\Tr(S^\sT\log S)
\Big\}\,.
\end{equation}
Here $\log$ is applied entrywise, and for the mixture-weight matrix $S$ whose $i$-th row is $s_i^\sT$, we have $\Tr(S^\sT\log S)=\sum_{i=1}^n s_i^\sT\log s_i$. The loss function $\ell_n(.)$ is convex in $S$ for fixed $Z$ and convex in $Z$ for fixed $S$, but it is not jointly convex in $(S,Z)$, owing to the multiplicative coupling $SZ^\sT$. The entropic term $\tfrac{\tau}{n}\Tr(S^\sT\log S)$ (with $\tau>0$) is strictly convex in $S$ and its gradient diverges at the boundary of the simplex, so it keeps the weights in the relative interior of $\Delta_{K-1}$ and makes the objective strongly convex in $S$ for each fixed $Z$. The optimal weights are then unique, which is especially useful in the overparametrized regime $K>d$, where the unregularized problem need not be identifiable.

For a fixed value of $\tau$, we let $\cC_n$ denote the set of first-order stationary points of the optimization problem~\eqref{eq:context-free}. More precisely, we define
\begin{equation}\label{eq:stationarity}
\cC_n:=\Bigl\{(\hat S,\hat Z)\in\cS\times\cZ:\ \nabla_S\ell_n(\hat S,\hat Z)\cdot(S-\hat S)+\nabla_Z\ell_n(\hat S,\hat Z)\cdot(Z-\hat Z)\ge0\ \text{ for all }S\in\cS,\ Z\in\cZ\Bigr\}.
\end{equation}
In this formulation, the set $\cC_n$ is nonempty, as $\ell_n$ is continuous on the compact set $\cS\times\cZ$ and hence attains its minimum, which is first-order stationary. 

In the hard-constraint learning problem with covariates $x_i\in \reals^p$, we focus on the class of two-layer neural networks with activation function
$\sigma:\mathbb{R}\to\mathbb{R}$. For $N$ neurons, first-layer
parameters $U\in\mathbb{R}^{p\times N}$, output-layer parameters
$\th\in\mathbb{R}^{N\times(K-1)}$, and $K$ prototypes
$Z=(z_1,\ldots,z_{K})\in\mathbb{R}^{d\times K}$, with
$z_1,\ldots,z_{K}\in\mathcal{P}$, we define the following mapping
\begin{equation}\label{eq:neural-network}
f(x;U,\th,Z)
=
Z\,\mathsf{soft}\big(\th^\sT\sigma(U^\sT x)\big)\,.
\end{equation}
Here, $\sigma(\cdot)$ is applied componentwise. We next focus on the iterative update of these network parameters, where at iteration $t$ we let $\beta_t=(U_t,\th_t,Z_t)$, and for the weight matrix $S_t=\mathsf{soft}(\sigma(XU_t)\th_t)$ with $\mathsf{soft}$ applied row-wise, mapping each row of $\sigma(XU_t)\th_t\in\reals^{n\times(K-1)}$ to the corresponding row of $S_t$, we consider the loss function $\ell_n(S_t,Z_t)$ as in \eqref{eq:context-free}. For step sizes $\eta_t$, we consider the following gradient descent updates:
\begin{align}\label{eq:iterates}
(U_{t+1},\th_{t+1})
&=(U_t,\th_t)-\eta_t\nabla_{U,\th}\ell_n\big(\mathsf{soft}(\sigma(XU_t)\th_t),Z_t\big)\,, \nonumber\\
Z_{t+1}
&=\mathsf{proj}_{\cZ}\Big(Z_t-\eta_t \nabla_{Z}\ell_n\big(\mathsf{soft}(\sigma(XU_t)\th_t),Z_t\big)\Big) \,.
\end{align}

In this section, we show that our prototype-based architecture generates a sequence of mixture weights $S_t=\mathsf{soft}(\sigma(XU_t)\th_t)$ and prototypes $Z_t$ that becomes arbitrarily close to the first-order stationary points of the context-free clustering problem \eqref{eq:context-free}. This implies that replacing the freely chosen mixture weights with our shared context-dependent neural network architecture does not prevent the learned mixture weights and prototypes from approaching first-order stationary points of the context-free clustering problem. Before the formal statement of the theorem, we mention a few assumptions.

\begin{assumption}
\label{assumption:activation}
The activation function $\sigma:\reals\to\reals$ is twice continuously differentiable and, writing $\sigma^{(j)}$ for its $j$-th derivative with $\sigma^{(0)}=\sigma$, there exists a constant $A>0$ such that
\[
\max_{j\in\{0,1,2\}}\sup_{v\in\reals}\bigl|\sigma^{(j)}(v)\bigr|\le A\,.
\]
For example, the sigmoid and hyperbolic tangent activation functions satisfy this assumption.
\end{assumption}

\begin{assumption}
\label{assumption:bounded-iterates}
There exists a constant $B<\infty$ such that
\[
\|U_t\|_{\mathrm{F}}\le B
\quad\text{and}\quad
\|\th_t\|_{\mathrm{F}}\le B,
\qquad\forall\,t\ge0.
\]
\end{assumption}
In fact, this boundedness can be enforced directly by projecting the first-layer parameters $(U_t,\th_t)$ onto the ball of radius $B$ at each step, and all subsequent results continue to hold.  We nonetheless keep it as an assumption for simplicity.

\begin{theorem}\label{thm:stationary-points}
Consider the prototype-based neural network for hard-constraint output learning given in \eqref{eq:neural-network}, with an activation function satisfying Assumptions~\ref{assumption:activation} and~\ref{assumption:bounded-iterates}. Suppose further that the first-layer features are well-conditioned, i.e., there exists a constant $c_{\min}>0$ such that
\[
\sigma(XU_t)\sigma(XU_t)^\sT
\succeq c_{\min}I_n,
\qquad \forall\,t\ge0.
\]
If the diminishing step sizes $\eta_t>0$ satisfy the standard conditions\footnote{That is, $\sum_{t\ge0}\eta_t=\infty$ and $\sum_{t\ge0}\eta_t^2<\infty$.}, then the iterates approach the set $\cC_n$ of first-order stationary points defined in \eqref{eq:stationarity}, i.e.,
\[
\lim\limits_{t\to \infty}\mathsf{dist}\bigl((S_t,Z_t),\cC_n\bigr)= 0\,.
\]
\end{theorem}

The lower bound $\sigma(XU_t)\sigma(XU_t)^\sT\succeq c_{\min}I_n$ forces the $n\times N$ feature matrix $\sigma(XU_t)$ to have full row rank uniformly in $t$, which is only possible when there are at least as many neurons as samples, i.e., in the overparameterized regime $N\ge n$. 

\begin{corollary}\label{cor:isolated}
Suppose the assumptions of Theorem~\ref{thm:stationary-points} hold and, in addition, that the stationary points in $\cC_n$ are isolated, i.e., for every $(\hat S,\hat Z)\in\cC_n$ there exists $\rho>0$ such that the ball $B_\rho(\hat S,\hat Z)$ satisfies $B_\rho(\hat S,\hat Z)\cap\cC_n=\{(\hat S,\hat Z)\}$. Then the iterates of the prototype architecture converge to a single first-order stationary point. More precisely, there exists $(\hat S,\hat Z)\in\cC_n$ with $(S_t,Z_t)\to(\hat S,\hat Z)$, as $t\to\infty$.
\end{corollary}
\begin{proof}
We first show that $\cC_n$ is nonempty and compact. Non-emptiness follows since $\ell_n$ is continuous on the compact set $\cS\times\cZ$ and therefore attains its minimum, and any minimizer satisfies the first-order optimality condition~\eqref{eq:stationarity} and so it belongs to $\cC_n$. For compactness, we recall that the set of first-order stationary points of a continuously differentiable function over a closed convex set is itself closed, since the defining first-order optimality conditions are preserved under limits~\citep{Bertsekas1997}. In addition, as $\cC_n\subseteq\cS\times\cZ$ is a closed subset of the compact set $\cS\times\cZ$, it is bounded and hence compact. Since $\cC_n$ is compact with isolated points, it is finite. This holds mainly, as around each point of $\cC_n$ there is an open ball containing no other point of $\cC_n$. Therefore, these balls form an open cover of $\cC_n$, and by compactness a finite subcover suffices, so $\cC_n$ can contain only finitely many points. We claim that the moving parts of the iteration also go to zero. More precisely, we claim that $\|S_{t+1}-S_t\|_{\mathrm{F}}\to0$ and $\|Z_{t+1}-Z_t\|_{\mathrm{F}}\to0$. We prove this claim later, but use it for now. Recalling $\mathsf{dist}\big((S_t,Z_t),\cC_n\big)\to0$ from Theorem~\ref{thm:stationary-points}, the iterates must converge to a single stationary point. This holds as distinct points of $\cC_n$ are separated by a positive distance, so once the iterates are close to $\cC_n$ and their increments are small, they can no longer move from the vicinity of one stationary point to that of another, and must converge to one point. 

It remains to establish the claim. For the prototypes, $\|Z_{t+1}-Z_t\|_{\mathrm{F}}=\eta_t\big(\|Z_{t+1}-Z_t\|_{\mathrm{F}}/\eta_t\big)\to0$, since $\eta_t\to0$ and $\|Z_{t+1}-Z_t\|_{\mathrm{F}}/\eta_t\to0$ by Lemma~\ref{lem:vanish}. For the weights, the first-layer update gives $\|(U_{t+1},\th_{t+1})-(U_t,\th_t)\|_{\mathrm{F}}=\eta_t\|\nabla_{U,\th}\ell_n(S_t,Z_t)\|_{\mathrm{F}}\to0$ by Lemma~\ref{lem:vanish}. In addition, since $S_t=\mathsf{soft}(\sigma(XU_t)\th_t)$ is a continuously differentiable function of $(U_t,\th_t)$ (see the proof of Lemma~\ref{lem:smooth}), hence Lipschitz on the compact set $\mathcal K$, it follows that $\|S_{t+1}-S_t\|_{\mathrm{F}} \to0$ as well.
\end{proof}
We emphasize that isolation of the stationary points is a property of the data distribution, not of the learning architecture. This implies that, whenever this isolation property holds, the prototype-based iterates inherit convergence to a single stationary point.
\subsection{Proof of Theorem~\ref{thm:stationary-points}}
We first state the following two lemmas, whose proofs are deferred to Sections~\ref{subsec:proof-smooth} and~\ref{subsec:proof-vanish}. The first establishes the Lipschitzness of the loss $\ell_n$ on $\mathcal{K}$, and the second the vanishing of its gradient with respect to $(U,\th)$ and to $Z$.

\begin{lemma}\label{lem:smooth}
Under Assumptions~\ref{assumption:activation} and \ref{assumption:bounded-iterates}, the composed loss $(U,\th,Z)\mapsto\ell_n(\mathsf{soft}(\sigma(XU)\th),Z)$ is twice continuously differentiable, and on the compact convex set
\[
\mathcal{K}:=\{U:\|U\|_{\mathrm{F}}\le B\}\times\{\th:\|\th\|_{\mathrm{F}}\le B\}\times\cZ
\]
its gradient is $L$-Lipschitz, with $L$ the maximum of the Hessian norm of the composed loss over $\mathcal{K}$. Here $L$-Lipschitzness holds jointly in the parameters $\beta=(U,\th,Z)$, that is, $\|\nabla\ell_n(\beta)-\nabla\ell_n(\beta')\|_{\mathrm{F}}\le L\|\beta-\beta'\|_{\mathrm{F}}$ for all $\beta,\beta'\in\mathcal{K}$.

\end{lemma}

\begin{lemma}\label{lem:vanish}
Consider the iterates $(U_t,\th_t,Z_t)$ generated by \eqref{eq:iterates}, and write $S_t=\mathsf{soft}(\sigma(XU_t)\th_t)$. Under Assumptions~\ref{assumption:activation} and \ref{assumption:bounded-iterates}, and diminishing step sizes $\eta_t>0$ satisfying $\sum_{t\ge0}\eta_t=\infty$ and $\sum_{t\ge0}\eta_t^2<\infty$, the iterates satisfy
\[
\lim\limits_{t\to \infty}\big\|\nabla_{U,\th}\ell_n(S_t,Z_t)\big\|_{\mathrm{F}}= 0\,,
\qquad
\lim\limits_{t\to \infty}\frac{\|Z_{t+1}-Z_t\|_{\mathrm{F}}}{\eta_t}= 0\,.
\]

\end{lemma}


We now focus on proving Theorem~\ref{thm:stationary-points}. We start by defining the gradient matrix $G\in\reals^{n\times(K-1)}$ as a function of $(S,Z)$ by
\[
G_{ir}
=
S_{ir}\Big(
\bigl(\nabla_S\ell_n(S,Z)\bigr)_{ir}
-\sum_{q=1}^{K}S_{iq}\bigl(\nabla_S\ell_n(S,Z)\bigr)_{iq}
\Big),
\qquad \forall\,i\in[n],\ r\in[K-1].
\]
We let $G_t$ denote the matrix $G$ defined above, evaluated at $(S_t,Z_t)$. Chain rule gives us
\[
\nabla_\th\ell_n(S_t,Z_t)
=
\sigma(XU_t)^\sT G_t.
\]
Therefore, the well-conditioning assumption gives
\[
\left\|\nabla_\th\ell_n(S_t,Z_t)\right\|_{\mathrm{F}}^2
=\Tr\left(
G_t^\sT\sigma(XU_t)\sigma(XU_t)^\sT G_t
\right)
\ge c_{\min}\|G_t\|_{\mathrm{F}}^2\,.
\]
Finally, Lemma~\ref{lem:vanish} gives us $\|G_t\|_{\mathrm{F}}\to0$.

Now consider any convergent subsequence $(S_{t_j},Z_{t_j})\to(\hat S,\hat Z)$\footnote{Since $\cS\times\cZ$ is compact, every sequence in it admits a subsequence converging to a point of $\cS\times\cZ$.}. The boundedness of the logits implies that every entry of $\hat S$ is strictly positive. Taking limits in the definition of $G_t$, for every $i\in[n]$ and $r\in[K-1]$ we obtain
\[
\bigl(\nabla_S\ell_n(\hat S,\hat Z)\bigr)_{ir}
=
\sum_{q=1}^{K}\hat S_{iq}
\bigl(\nabla_S\ell_n(\hat S,\hat Z)\bigr)_{iq}.
\]
Since the right-hand side does not depend on $r$ and $\hat S_{iK}>0$, the same equality also holds for $r=K$. Thus, each row of $\nabla_S\ell_n(\hat S,\hat Z)$ is constant, which gives
\[
\nabla_S\ell_n(\hat S,\hat Z)\cdot(S-\hat S)=0,
\qquad \forall\,S\in\cS.
\]
In addition, the optimality condition of the $\cZ$-projection gives
\[
\Big(
\nabla_Z\ell_n(S_t,Z_t)
-\frac{Z_t-Z_{t+1}}{\eta_t}
\Big)\cdot(Z-Z_{t+1})\ge0,
\qquad \forall\,Z\in\cZ.
\]
Taking limits along the same subsequence and using
$\|Z_{t+1}-Z_t\|_{\mathrm{F}}/\eta_t\to0$, we obtain
\[
\nabla_Z\ell_n(\hat S,\hat Z)\cdot(Z-\hat Z)\ge0,
\qquad \forall\,Z\in\cZ.
\]
Combining the last two displays shows that $(\hat S,\hat Z)\in\cC_n$. 

Finally, we show that $\limsup_{t\to\infty}\mathsf{dist}\bigl((S_t,Z_t),\cC_n\bigr)=0$. Suppose this does not happen. It means that for some $\eps>0$ there exists a subsequence satisfying $\mathsf{dist}((S_{t_j},Z_{t_j}),\cC_n)\ge\eps$. Since $\cS\times\cZ$ is compact, this subsequence has a further convergent subsequence. But from the previous argument we know that this limiting point must belongs to $\cC_n$. As $\mathsf{dist}(\cdot,\cC_n)$ is continuous, the distance along it tends to $0$, contradicting the lower bound $\varepsilon$. Hence $\lim\sup_{t\to\infty}\mathsf{dist}\bigl((S_t,Z_t),\cC_n\bigr)=0$. On the other hand,  $\mathsf{dist}(\cdot,\cC_n)$ is nonnegative, this implies that
\[
0\le\lim\inf_{t\to\infty}\mathsf{dist}\bigl((S_t,Z_t),\cC_n\bigr)\le\lim\sup_{t\to\infty}\mathsf{dist}\bigl((S_t,Z_t),\cC_n\bigr)=0,
\]
so the limit exists and therefore 
$\mathsf{dist}\bigl((S_t,Z_t),\cC_n\bigr)\to 0$, as $t\to \infty$.

\subsection{Proof of Lemma~\ref{lem:smooth}}\label{subsec:proof-smooth}
We claim that the following loss function is two-times continuously differentiable everywhere:
\[
(U,\th,Z)\to \ell_n\bigl(\mathsf{soft}(\sigma(XU)\th),Z\bigr)\,.
\]
To show this, we start with the map $(U,\th)\mapsto\sigma(XU)\th$. The map $U\mapsto XU$ is linear, and $\sigma$ is twice continuously differentiable by Assumption~\ref{assumption:activation}, so the composition $\sigma(XU)$ is twice continuously differentiable in $U$. Multiplying by $\th$ is again a linear operation, so $\sigma(XU)\th$ is twice continuously differentiable jointly in $(U,\th)$. We next focus on $\mathsf{soft}$. Each coordinate of $\mathsf{soft}$ has the form $e^{v_r}/(1+\sum_{q}e^{v_q})$, and the exponential function is smooth everywhere and strictly positive at every finite point, so the denominator is always strictly positive and $\mathsf{soft}$ is twice continuously differentiable on all of $\reals^{K-1}$. Composing this with the map from the previous paragraph shows that $S=\mathsf{soft}(\sigma(XU)\th)$ is twice continuously differentiable in $(U,\th)$.

It remains to deal with the $\log S$ term appearing in $\ell_n$. Using the same expression for $\mathsf{soft}$ as a ratio of exponentials, every coordinate of $S$ is strictly positive whenever $\sigma(XU)\th$ is finite. Since $\log$ is twice continuously differentiable on the positive reals, this shows that $\log S$ is twice continuously differentiable at every finite point $(U,\th)$ as well.

The remaining term in $\ell_n$, namely $\|SZ^\sT-Y\|_{\mathrm{F}}^2$, is a polynomial in $S$ and $Z$ and is therefore smooth. Combining this with the previous two paragraphs, we conclude that the composed map $(U,\th,Z)\mapsto\ell_n\bigl(\mathsf{soft}(\sigma(XU)\th),Z\bigr)$ is two-times continuously differentiable at every finite point $(U,\th,Z)$.

In the next step, by Assumption~\ref{assumption:bounded-iterates}, the trajectory $(U_t,\th_t,Z_t)$ stays, for every $t$, in the fixed compact convex set
\[
\mathcal{K}:=\{U:\|U\|_{\mathrm{F}}\le B\}\times\{\th:\|\th\|_{\mathrm{F}}\le B\}\times\cZ.
\]
Since the composed loss is two-times continuously differentiable on all of $\mathcal{K}$, its Hessian is a continuous function of $(U,\th,Z)$, and a continuous function attains a finite maximum on a compact set. We denote this maximum by $L$. It brings us that the gradient of the $\ell_n$ loss is $L$-Lipschitz on $\mathcal{K}$.

\subsection{Proof of Lemma~\ref{lem:vanish}}\label{subsec:proof-vanish}
Throughout, $L$ and $\mathcal{K}$ denote the Lipschitz constant and compact convex set introduced in Lemma~\ref{lem:smooth}.

In the next step, since $\eta_t\to0$ by the step-size conditions, we have $\eta_t\le 1/L$ for all sufficiently large $t$. Because the gradient of the composed loss is $L$-Lipschitz on $\mathcal{K}$, for any two points $\beta,\beta'\in\mathcal{K}$, the first order expansion gives
\[
\ell_n(\beta')\le\ell_n(\beta)+\nabla\ell_n(\beta)\cdot(\beta'-\beta)+\frac{L}{2}\|\beta'-\beta\|_{\mathrm{F}}^2.
\]

We first apply this to the $(U,\th)$-update. Since $(U_{t+1},\th_{t+1})-(U_t,\th_t)=-\eta_t\nabla_{U,\th}\ell_n(\mathsf{soft}(\sigma(XU_t)\th_t),Z_t)$, plugging $\beta=(U_t,\th_t,Z_t)$ and $\beta'=(U_{t+1},\th_{t+1},Z_t)$ into the above gives
\begin{align*}
\ell_n(\mathsf{soft}(\sigma(XU_{t+1})\th_{t+1}),Z_t)
&\le
\ell_n(\mathsf{soft}(\sigma(XU_t)\th_t),Z_t)
-\eta_t\left\|\nabla_{U,\th}\ell_n(\mathsf{soft}(\sigma(XU_t)\th_t),Z_t)\right\|_{\mathrm{F}}^2\\
&\qquad+\frac{L\eta_t^2}{2}\left\|\nabla_{U,\th}\ell_n(\mathsf{soft}(\sigma(XU_t)\th_t),Z_t)\right\|_{\mathrm{F}}^2\\
&=
\ell_n(\mathsf{soft}(\sigma(XU_t)\th_t),Z_t)
-\eta_t\left(1-\frac{L\eta_t}{2}\right)\left\|\nabla_{U,\th}\ell_n(\mathsf{soft}(\sigma(XU_t)\th_t),Z_t)\right\|_{\mathrm{F}}^2\\
&\le
\ell_n(\mathsf{soft}(\sigma(XU_t)\th_t),Z_t)
-\frac{\eta_t}{2}\left\|\nabla_{U,\th}\ell_n(\mathsf{soft}(\sigma(XU_t)\th_t),Z_t)\right\|_{\mathrm{F}}^2,
\end{align*}
where the last step uses $\eta_t\le1/L$, so that $1-L\eta_t/2\ge1/2$. Here $\mathsf{soft}(\sigma(XU_{t+1})\th_{t+1})$ on the left denotes the mixture weights obtained after the $(U,\th)$-update alone, before $Z_t$ is updated to $Z_{t+1}$.

We next bring in the $Z$-update. Since $Z_{t+1}=\mathsf{proj}_\cZ(Z_t-\eta_t\nabla_Z\ell_n(\mathsf{soft}(\sigma(XU_t)\th_t),Z_t))$ and $\cZ$ is convex, the standard optimality condition of the Euclidean projection gives
\[
\nabla_Z\ell_n(\mathsf{soft}(\sigma(XU_t)\th_t),Z_t)\cdot(Z_{t+1}-Z_t)
\le
-\frac{1}{\eta_t}\|Z_{t+1}-Z_t\|_{\mathrm{F}}^2.
\]

We next use the first order expansion along with the $L$-smoothness of $\ell_n(.)$, and combining it with the above projection inequality to arrive at
\begin{align*}
\ell_n(\mathsf{soft}(\sigma(XU_{t+1})\th_{t+1}),Z_{t+1})
&\le
\ell_n(\mathsf{soft}(\sigma(XU_{t+1})\th_{t+1}),Z_t)
+\left(\frac{L}{2}-\frac{1}{\eta_t}\right)\|Z_{t+1}-Z_t\|_{\mathrm{F}}^2\\
&\le
\ell_n(\mathsf{soft}(\sigma(XU_{t+1})\th_{t+1}),Z_t)
-\frac{1}{2\eta_t}\|Z_{t+1}-Z_t\|_{\mathrm{F}}^2,
\end{align*}
where the last step again uses $\eta_t\le1/L$. Combining the two displays above, for all sufficiently large $t$ we obtain
\begin{align*}
\ell_n(\mathsf{soft}(\sigma(XU_{t+1})\th_{t+1}),Z_{t+1})
&\le
\ell_n(\mathsf{soft}(\sigma(XU_t)\th_t),Z_t)
-\frac{\eta_t}{2}
\left\|\nabla_{U,\th}
\ell_n\bigl(\mathsf{soft}(\sigma(XU_t)\th_t),Z_t\bigr)\right\|_{\mathrm{F}}^2\\
&\;-\frac{1}{2\eta_t}\|Z_{t+1}-Z_t\|_{\mathrm{F}}^2.
\end{align*}

In the next step, let $S_t=\mathsf{soft}(\sigma(XU_t)\th_t)$, and $t_0$ be such that $\eta_t\le1/L$ for every $t\ge t_0$, so that the above inequality holds for all such $t$.

Summing the above inequality over $t=t_0,\ldots,T-1$, brings us
\[
\ell_n(S_T,Z_T)
+\frac12\sum_{t=t_0}^{T-1}\eta_t\left\|\nabla_{U,\th}\ell_n(S_t,Z_t)\right\|_{\mathrm{F}}^2
+\frac12\sum_{t=t_0}^{T-1}\frac{\|Z_{t+1}-Z_t\|_{\mathrm{F}}^2}{\eta_t}
\le
\ell_n(S_{t_0},Z_{t_0}).
\]
Since each row $s_i\in\Delta_{K-1}$ is a probability vector, $\sum_{r}s_{ir}\log s_{ir}\ge-\log K$, and since $\|SZ^\sT-Y\|_{\mathrm{F}}^2\ge0$, this gives $\ell_n(S,Z)\ge-\tau\log K$ for every $(S,Z)\in\cS\times\cZ$. Applying this to $\ell_n(S_T,Z_T)$ and rearranging, the right side above is bounded by $\ell_n(S_{t_0},Z_{t_0})+\tau\log K$, a quantity independent of $T$. Letting $T\to\infty$ gives
\begin{equation}
\sum_{t\ge0}\eta_t
\left\|\nabla_{U,\th}\ell_n(S_t,Z_t)\right\|_{\mathrm{F}}^2<\infty,
\qquad
\sum_{t\ge0}\frac{\|Z_{t+1}-Z_t\|_{\mathrm{F}}^2}{\eta_t}<\infty.
\end{equation}

In the next step, we show that the gradient value $\nabla_{U,\th}\ell_n(.)$ and $\|Z_{t+1}-Z_t\|_{\mathrm{F}}/\eta_t$ converge to zero as the iterations $T$ grows to $\infty$. For simplicity, let $a_t=\left\|\nabla_{U,\th}\ell_n(S_t,Z_t)\right\|_{\mathrm{F}}$ and $b_t=\|Z_{t+1}-Z_t\|_{\mathrm{F}}/\eta_t$, we are interested in establishing $a_t,b_t$ converge to zero. For this end, let $d_t=(a_t^2+b_t^2)^{1/2}$. Using the above two bounds it is easy to observe that $\sum_{t\ge0}\eta_t d_t^2<\infty\,.$ From the definitions, it is straightforward to see that
\[
\|(U_{t+1},\th_{t+1},Z_{t+1})-(U_t,\th_t,Z_t)\|_{\mathrm{F}}=\eta_t d_t\,,
\]
We next deploy the $L$-Lipschitzness of $\ell_n(.)$ in $(U,\th, Z)$ on $\mathcal{K}$. Specifically, using the triangle inequality this gives us
\[
|a_{t+1}-a_t|
\le\bigl\|\nabla_{U,\th}\ell_n(S_{t+1},Z_{t+1})-\nabla_{U,\th}\ell_n(S_t,Z_t)\bigr\|_{\mathrm{F}}
\le L\eta_t d_t\,.
\]
we first recall this property that if $v=\mathsf{proj}_{\cZ}(u)$, then the optimality condition of $\ell_2$-norm projection gives us that for every $z\in \cZ$ we have
$\langle v-u,z-v\rangle\ge0$. 
We apply this condition for the following projection steps $\widehat Z_{t+2}:=
\mathsf{proj}_{\cZ}\!\left(
Z_{t+1}-\eta_{t+1}\nabla_Z\ell_n(S_t,Z_t)
\right)$ and  $Z_{t+1}
=
\mathsf{proj}_{\cZ}\!\left(
Z_t-\eta_t\nabla_Z\ell_n(S_t,Z_t)
\right)$. More precisely, for every $Z\in \cZ$ we obtain 
\[
\Big\langle
\nabla_Z\ell_n(S_t,Z_t)-\frac{Z_t-Z_{t+1}}{\eta_t},
Z-Z_{t+1}
\Big\rangle\ge0\,, \quad 
\Big\langle
\nabla_Z\ell_n(S_t,Z_t)
-\frac{Z_{t+1}-\widehat Z_{t+2}}{\eta_{t+1}},
Z-\widehat Z_{t+2}
\Big\rangle\ge0\,.
\]
 We next use $Z=\widehat Z_{t+2}$ in the first inequality and $Z=Z_{t+1}$ in the second. This brings us
\[
\frac{1}{\eta_{t+1}}\|
{Z_{t+1}-\widehat Z_{t+2}}
\|_{\mathrm{F}}^2
\le
\frac{1}{\eta_t}\langle
{Z_t-Z_{t+1}},
{Z_{t+1}-\widehat Z_{t+2}}
\rangle.
\]
Finally, using Cauchy's inequality we arrive at
\[
\frac{1}{\eta_{t+1}}\|
{Z_{t+1}-\widehat Z_{t+2}}
\|_{\mathrm{F}}
\le
\frac{1}{\eta_t}\|{Z_t-Z_{t+1}}\|_{\mathrm{F}}
=b_t.
\]
Substituting the definition of $\widehat Z_{t+2}$, we obtain
\[
\frac{1}{\eta_{t+1}}
\left\|Z_{t+1}
-\mathsf{proj}_{\cZ}\!\left(
Z_{t+1}-\eta_{t+1}\nabla_Z\ell_n(S_t,Z_t)
\right)\right\|_{\mathrm{F}}
\le b_t\,.
\]
We next note that we have
\begin{align*}
b_{t+1}&=\frac{1}{\eta_{t+1}}
\left\|Z_{t+1}
-\mathsf{proj}_{\cZ}\!\left(
Z_{t+1}-\eta_{t+1}\nabla_Z\ell_n(S_{t+1},Z_{t+1})
\right)\right\|_{\mathrm{F}}\\
&\le \frac{1}{\eta_{t+1}}
\left\|Z_{t+1}
-\mathsf{proj}_{\cZ}\!\left(
Z_{t+1}-\eta_{t+1}\nabla_Z\ell_n(S_{t},Z_{t})
\right)\right\|_{\mathrm{F}} + \left\|\nabla_Z\ell_n(S_{t+1},Z_{t+1})
-\nabla_Z\ell_n(S_t,Z_t)\right\|_{\mathrm{F}}\,,
\end{align*}
where the last relation follows from the $\alpha$-Lipschitzness of the mapping $g\to u-\mathsf{proj}_{\cZ}(v-\alpha g)$.
Finally, combining the above relations give us
\[
\begin{aligned}
b_{t+1}
&\le b_t+
\left\|\nabla_Z\ell_n(S_{t+1},Z_{t+1})
-\nabla_Z\ell_n(S_t,Z_t)\right\|_{\mathrm{F}}\\
&\le b_t+L\eta_t d_t\,.
\end{aligned}
\]
In summary, we got $a_{t+1}\le a_t+L\eta_t d_t$, $b_{t+1}\le b_t+L\eta_t d_t.$ We next use $a_t,b_t\le d_t$ and the boundedness of $(\eta_t)$, this yields, for some constant $C>0$,
we have $a_{t+1}^2\le a_t^2+C\eta_t d_t^2$ and 
$b_{t+1}^2\le b_t^2+C\eta_t d_t^2$.

We next focus on proving $a_t,b_t\to 0$. For this goal, let $R_t=C\sum_{j=t}^{\infty}\eta_jd_j^2$. Since $\sum_{t\ge0}\eta_td_t^2<\infty$, we have this tail sum $R_t\to0$. In addition, from the definitions, we have
\[
a_{t+1}^2+R_{t+1}\le a_t^2+C\eta_td_t^2+R_{t+1}=a_t^2+R_t\,.
\]
This implies that  $a_t^2+R_t$ is a non-increasing sequence which is also bounded below, implying that it converges. Since $R_t\to0$, therefore $a_t^2$ also converges. The same argument shows that $b_t^2$ converges. We just left to show that the converging values are zero.  Let $a_t^2\to A$ and $b_t^2\to B$, where $A,B\ge0$. If $A>0$, then $a_t^2\ge A/2$ for all sufficiently large $t$, implying $\sum_t\eta_ta_t^2=\infty$ because $\sum_t\eta_t=\infty$, which contradicts $\sum_t\eta_td_t^2<\infty$. Thus $A=0$, and the same argument gives $B=0$. This completes the proof.

\end{document}